\documentclass{article}

    \PassOptionsToPackage{numbers, compress}{natbib}

\usepackage[final]{neurips_2024}

\usepackage[utf8]{inputenc} 
\usepackage[T1]{fontenc}    
\usepackage{hyperref}       
\usepackage{url}            
\usepackage{booktabs}       
\usepackage{amsfonts}       
\usepackage{nicefrac}       
\usepackage{microtype}      
\usepackage{xcolor}         
\usepackage{amsmath,amstext,amsfonts,bm,amssymb,mathtools,amsthm}
\usepackage{tcolorbox}
\usepackage{color,xcolor,xspace}
\usepackage{booktabs}
\usepackage{thm-restate}
\usepackage{bbding}
\usepackage{pifont}
\usepackage{wasysym}
\usepackage{wrapfig}

\usepackage[inline]{enumitem}
\theoremstyle{plain}
\newtheorem{theorem}{Theorem}[section]
\newtheorem{proposition}[theorem]{Proposition}
\newtheorem{lemma}[theorem]{Lemma}

\theoremstyle{definition}
\newtheorem{definition}[theorem]{Definition}

\theoremstyle{remark}

\usepackage{booktabs}
\usepackage{colortbl}
\usepackage{placeins}
\definecolor{mygray}{gray}{.9}

\newcommand\blue{\textcolor{blue}}

\newcommand\red{\textcolor{red}}

\usepackage[capitalize]{cleveref}
\usepackage{natbib}

\usepackage{caption}

\title{Interpret Your Decision: Logical Reasoning Regularization for Generalization in Visual Classification}

\author{%
 Zhaorui Tan$^{1,2}$, Xi Yang$^1$$^*$, Qiufeng Wang$^1$, Anh Nguyen$^2$, Kaizhu Huang$^3$\thanks{Corresponding authors.}\\
  $^1$ Xi'an-Jiaotong Liverpool University\\
  $^2$ University of Liverpool \\
  $^3$Duke Kunshan University
}

\begin{document}

\maketitle

\begin{abstract}
  Vision models excel in image classification but struggle to generalize to unseen data, such as classifying images from unseen domains or discovering novel categories. In this paper, we explore the relationship between logical reasoning and deep learning generalization in visual classification. A logical regularization termed L-Reg is derived which bridges a logical analysis framework to image classification. 
  Our work reveals that L-Reg reduces the complexity of the model in
  terms of the feature distribution and classifier weights. 
  Specifically, we unveil the interpretability brought by L-Reg, as it enables the model to extract the salient features, such as faces to persons, for classification. 
  Theoretical analysis and experiments demonstrate that L-Reg enhances generalization across various scenarios, including multi-domain generalization and generalized category discovery. In complex real-world scenarios where images span unknown classes and unseen domains, L-Reg consistently improves generalization, highlighting its practical efficacy.
\end{abstract}

\section{Introduction}



\begin{wrapfigure}{r}{0.6\linewidth}
\vspace{-0.6cm}
  \centering
    \includegraphics[width=\linewidth]{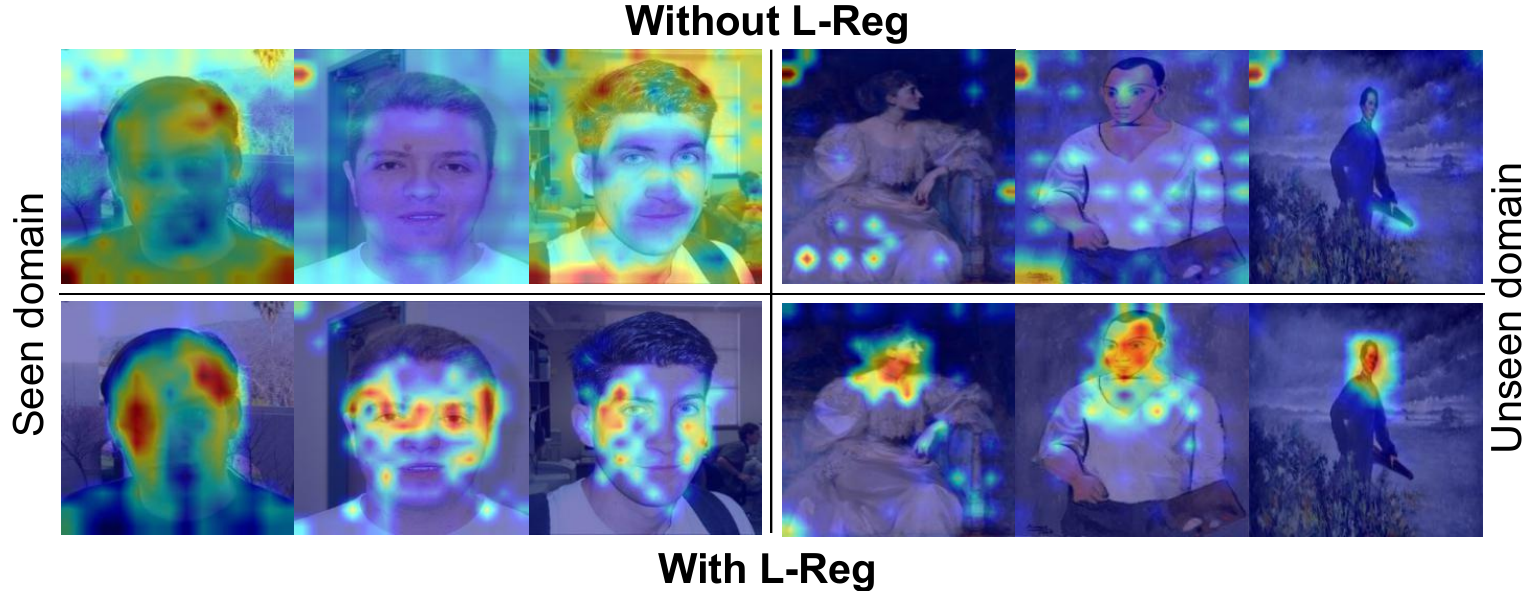}
  \caption{GradCAM~\cite{selvaraju2017grad} visualizations for the unknown class `person' across seen and unseen domains
  of {the GMDG baseline} with $L_2$ regularization that is trained without and with L-Reg, respectively. {Both experiments share the same hyper-parameters, except the latter uses the L-Reg.}}
  \label{fig:banner}
\vspace{-0.5cm}
\end{wrapfigure}
One critical challenge in visual classification models is their ability to generalize effectively to unseen samples or unknown classes. 
For instance, a model trained on real images of various animals should ideally classify animal sketches accurately (referred to as multi-domain generalization classification~\cite{ganin2016domain,li2018deep,li2018domain,hu2020domain,cha2022miro,li2022simple,tan2024rethinking}) or discover novel categories not present in the training set (referred to as generalized category discovery~\cite{vaze2022generalized,chiaroni2023parametric}). These problems are prevalent in real-world scenarios, where training data-target pairs are usually insufficient, and labeling is time-consuming so that not every data is paired with a label. Meanwhile,
test data is likely to contain shifts in both data and targets, making it essential to propose methods that generalize to border scenarios. 

Regularization terms, such as $L_2$ regularization leading to weight decay, are commonly employed during training to improve a model's generalization capabilities.
However,  the $L_2$ regularization is \textit{parametric-based} rather than \textit{sample-based}, which may lead to ambiguous interpretability~\cite{wu2018beyond}.
As illustrated in~\cref{fig:banner}, the model trained solely with $L_2$ regularization exhibits low interpretability. Other regularization terms~\cite{wu2017improving,wu2018beyond,wu2020regional} attempt to improve the interpretability of deep learning models for sequential signals rather than vision, 
whereas~\cite{moshe2022improving} proposes a regularization term to enhance interpretability for robustness in visual classification models rather than generalization. 
Drawing inspiration from logical reasoning has shown promise for better generalization and interpretability in various tasks. 
Current work unveils the effectiveness of logical reasoning in generalization tasks, such as boosting performance in length generalization~\cite{abbe2023generalization,ahuja2024provable,abbe2024provable,xiao2024theory} and abstract symbol relational reasoning~\cite{boix2023can,li2024neuro} (e.g., mathematical solving and psychological tests).
Several efforts, such as~\cite{barbiero2022entropy}, explore the explicit entropy-based logical explanations of neural networks for image classification, confirming the presence and interpretability of logical reasoning within visual tasks. 
Yet, there are limited studies tackling the generalization of visual classification tasks through the lens of logical reasoning.



This paper studies two pivotal questions corresponding to the above:
\textit{1) How does logical reasoning relate to visual tasks such as image classification? 2) How can we derive a logical reasoning-based regularization term to benefit generalization?}
To achieve these, we correlate the image classification procedure in computer vision with the framework of logic studies~\cite{andreka2017universal}, positing that training an image classifier involves learning a \textit{good general} logical relationship between images and labels via an encoder. This good general logic is attained when the semantics generated by the encoder and classifier can be combined to form {atomic formulas}. 
Our exploration leads to the introduction of a sample-based Logical regularization term named L-Reg.
We reveal that L-Reg efficiently reduces the \textit{complexity} of the model from two aspects: 
1) L-Reg leads to a balanced feature distribution in the semantic space;
2) L-Reg reduces the number of weights with extreme values in the classifier.

Intuitively, the complexity reduction achieved by L-Reg stems from its ability to filter out redundant features or semantics, focusing instead on the minimal yet sufficient semantics for classification - defined as semantic support in \cref{def:semantic_support}, where the interpretability also emerges.  
This filtering feature benefits the generalization when there is a domain shift in data where the domain-dependent features are ignored for classification. Moreover, it further promotes generalization when unlabeled data from the unknown classes is present. If such data lacks the semantic support associated with known classes, it is then classified as belonging to an unknown class, and its corresponding semantic supports are extracted.
These capabilities equip L-Reg with explicit interpretability. 
As \cref{fig:banner} shows, with L-Reg, the model can identify the unknown class `person', and pinpoint faces which are the crucial features for classifying this category.
In contrast, the model trained solely with $L_2$ (without L-Reg) focuses on the ambiguous features for classification.  

Rigorous theoretical analysis and experimental results validate that L-Reg yields better generalization across diverse scenarios. 
Specifically, 
L-Reg facilitates better performance under the aforementioned multi-domain generalization and generalized category discovery tasks, whose settings are presented in \cref{fig:settings}~(a)(b). 
Furthermore, to evaluate L-Reg's robustness, we introduce a more complex real-world scenario, as shown in \cref{fig:settings}~(c), where unlabeled images may not only belong to unknown classes, but also originate from unseen domains. Even in this challenging context, L-Reg is still able to consistently demonstrate notable improvements in generalization, underscoring its practical utility and effectiveness.
Our code is available at \url{https://github.com/zhaorui-tan/L-Reg_NeurIPS24}.




\section{Preliminaries and generalization settings for visual classification}\label{sec-setting}
\begin{wrapfigure}{r}{0.45\linewidth}
\vspace{-0.5cm}
  \centering
    \includegraphics[width=0.9\linewidth]{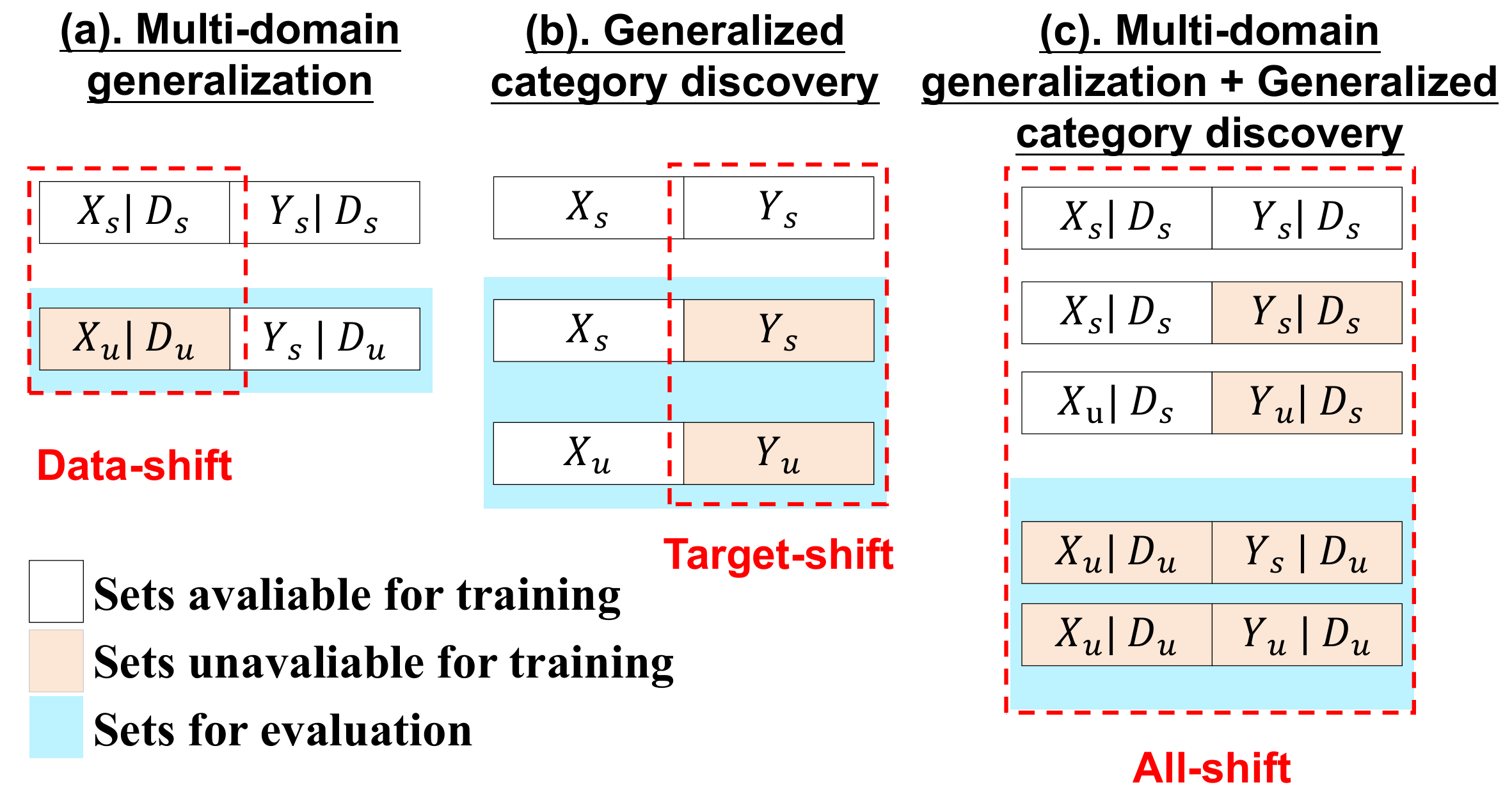}
  \captionof{figure}{Diagrams of different generalization settings in visual classification tasks.}
  \label{fig:settings}
\vspace{-0.2cm}
\end{wrapfigure}
Consider paired $(X, Y )\sim (\mathcal{X}, \mathcal{Y})$, $(X_s, Y_s) \sim (\mathcal{X}_s, \mathcal{Y}_s)$, and $(X_u, Y_u) \sim (\mathcal{X}_u, \mathcal{Y}_u)$ denote all sets of inputs and labels, seen paired subsets of $(X, Y)$, and unseen paired subsets of $(X, Y)$, respectively. 
Note that $X_u, Y_u$ may be accessible for the model separately, but their pairing relationships are not accessible. 
Let $D$ denote the possible domains, with $D_{s}, D_{u} \subset D$ representing the seen and unseen domains.
In classification tasks, an encoding function  $g(x) \to Z \in \mathbb{R}^M$ is commonly introduced to map  $X$ into the latent feature set $Z$, where each latent feature has $M$ dimensions.  A predictor $h(Z) \to \hat{Y} \in \mathbb{R}^K $ maps $Z$ to predictions $\hat{Y}$, where $K$ 
denotes the number of classes and the dimensions of predictions.
$P(\cdot)$ and $H(\cdot)$  symbolize probability and entropy, respectively.
This paper discusses two typical cases for generalization in image classification tasks: (1) \textit{Data-shift generalization:} $X_s$ and $X_u$ have distribution shifts, such as multi-domain generalization (mDG); and (2) \textit{Target-shift generalization:} $Y_s$ and $Y_u$ have distribution shifts, which stands for tasks like generalized category discovery (GCD). We additionally explore a challenging scenario called \textit{All-shift generalization:}  both $X_s$ and $X_u$, $Y_s$ and $Y_u$ have distribution shifts, which is a combination of mDG and GCD tasks (mDG + GCD).
The following lists the detailed settings for generalization. Please refer to \cref{fig:settings} for brief diagrams.

\textbf{Data-shift generalization: Problem setting for mDG.} 
Illustrated in \cref{fig:settings}~(a),
mDG~\citep{blanchard2011generalizing} intends to generalize well to unseen domains having the objective of $ \min H({X_s,Y_s \mid D_s} ) $ and expecting the model to be generalized to $X_u$ when predicting $Y_u$ from the unseen domain $D_u$.  In such cases, $Y_u$ is fully accessible to the model since $Y_s$ and $Y_u$ share the same domain: $\mathcal{Y}_s = \mathcal{Y}_u$ but there are shifts in $X$ where $\mathcal{X}_s \neq \mathcal{X}_u$.


\textbf{Target-shift generalization: Problem setting for GCD.}
GCD~\cite{vaze2022generalized} (\cref{fig:settings}~(b)) aims to discover possible unseen labels among unlabeled datasets $X_u$. The challenge is that 
the samples in $X_u$ may belong to known classes or unknown classes: $\mathcal{Y}_s \neq \mathcal{Y}_u$ and probably $\mathcal{Y}_s \cap \mathcal{Y}_u \neq \emptyset$. The model should be able to distinguish the samples from the known classes and cluster the samples for unknown classes simultaneously. Note that $X_u$ is used for model training, but the relationship between $X_u$ and $Y_u$ is unseen for the model. In summary, shifts exist between $Y_s$ and $Y_u$ but not between $X_s$ and $X_u$.

\textbf{All-shift generalization: Problem setting for mDG + GCD.} 
To explore the generalization problem further, we introduce a setting that is the combination of mDG and GCD as shown in \cref{fig:settings}~(c). Specifically, the model is trained on the labeled pairs $(X_s, Y_s)$ and unlabeled set $X_u$ from the seen domains $D_s$;  $X_u$ may belong to known and unknown classes. Furthermore,
the model is tested on $X_u$ from the unseen domain $D_u$, where $X_u$ may also come from the known and unknown classes.
In this setting, the model is expected to 1) classify samples to the seen classes and discover the unseen classes among unlabeled samples from seen domains and 2) generalize this ability to the samples from the unseen domain. In this scenario, $X_s$ and $X_u$ have shifts, and so do $Y_s$ and $Y_u$.

For all aforementioned generalization settings, the objective can be summarized as minimizing the \textit{generalization loss}:
\begin{definition}[Generalization loss] 
    Let the target model $f^*: f^*(X, Y): X \to Y$, can generalize across both seen and unseen sets $X, Y$. Denote its trainable $\red{f}$, which is only trained on the seen sets. The generalization loss for the unseen sets is defined as:
    \begin{equation}
        GL(\red{f}, f^*, (X_u, Y_u)) = \mathbb{E}_{(x, y)\in ({X}_u, {Y}_u)} || 
        f(x, y) - f^*(x, y)||_2.
    \end{equation}
\end{definition}





\begin{figure}[t]
  \centering
  \includegraphics[width=\linewidth]{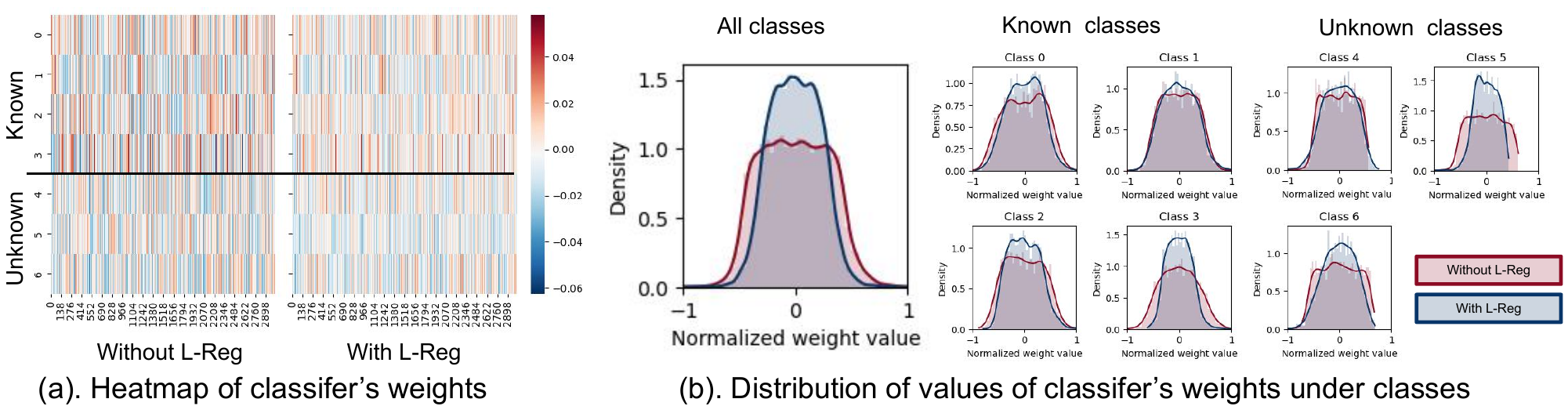}
    \captionof{figure}{Visualizations of classifiers' weights form models trained using GMDG on PACS dataset without and with L-Reg under mDG+GCD setting, respectively. Both experiments share the same hyper-parameters using Regnety-16g backbone, except the latter uses additional L-Reg. }
  \label{fig:classifier}
\end{figure}

\section{Logical regularization for generalization in image classification}
Under the problem settings defined in Section~\ref{sec-setting},
we introduce Logic regularization (L-Reg) targeting the objective:  
\begin{equation}
\label{eq:final_reg_obj}
    \min_{h,g} \mathbb{E}_{z_i \in z, z \in  Z} [H(\hat{Y}|z_i, D)] \! -\!
    \mathbb{E}_{ z \in  Z} [H(\hat{Y}|Z, D)]
    ,
\end{equation}
where $\hat{Y} \in \mathbb{R}^{K}= h\circ g (X)$ is the prediction set. 
The corresponding Logic regularization loss (L-Reg)  is defined as:
\begin{equation}
\label{eq:loss}
\begin{split}
        L_{L-Reg} = 
        & - \frac{1}{M}\sum_{i=1}^{M}\left[\sum_{j=1}^{K}\sigma_{j,i}(\hat{Y}^{T} Z)\log \sigma_{j,i}(\hat{Y}^{T} Z)\right] \\
        & + \sum_{j=1}^{K}\left[\frac{1}{M}\sum_{i=1}^{M}\sigma_{j,i}(\hat{Y}^{T} Z)\log(\frac{1}{M}\sum_{i=1}^{M}\sigma_{j,i}(\hat{Y}^{T} Z))\right]
        ,
\end{split}
\end{equation}
{where $\sigma_{j,i}(\hat{Y}^{T}Z)$ denotes the value at the $ i,j$ position of $softmax(\hat{Y}^{T}Z)$ and the soft-max function is applied at the last dimension. }
By incorporating other existing methods' losses denoted by $L_{main}$, the overall loss is formulated as: 
\begin{equation}
    L_{all} = L_{main} + \alpha L_{L-Reg},
\end{equation}
with a weight $\alpha$ applied to balance two losses.
As depicted in \cref{fig:banner}, L-Reg plays a pivotal role in extracting crucial features for image classification, thus enhancing generalization capabilities. This beneficial outcome can be attributed to two primary factors: 

\textbf{Reducing classifier complexity:} 
L-Reg streamlines the complexity of the classifier itself, as depicted in \cref{fig:classifier}~(a).
Notably, the heat map of the model with L-Reg displays fewer extremely valued weights, evidenced by the diminished presence of intense blue and red colors.
This reduction implies that the classifier focuses on leveraging semantically rich and relevant features for decision-making (classification), sidelining the less relevant ones. 
Additionally, \cref{fig:classifier}~(b) reveals a reduction in the number of semantic features used to classify each class. 

\begin{wrapfigure}{r}{0.4\linewidth}
\vspace{-0.5cm}
  \centering
  \includegraphics[width=\linewidth]{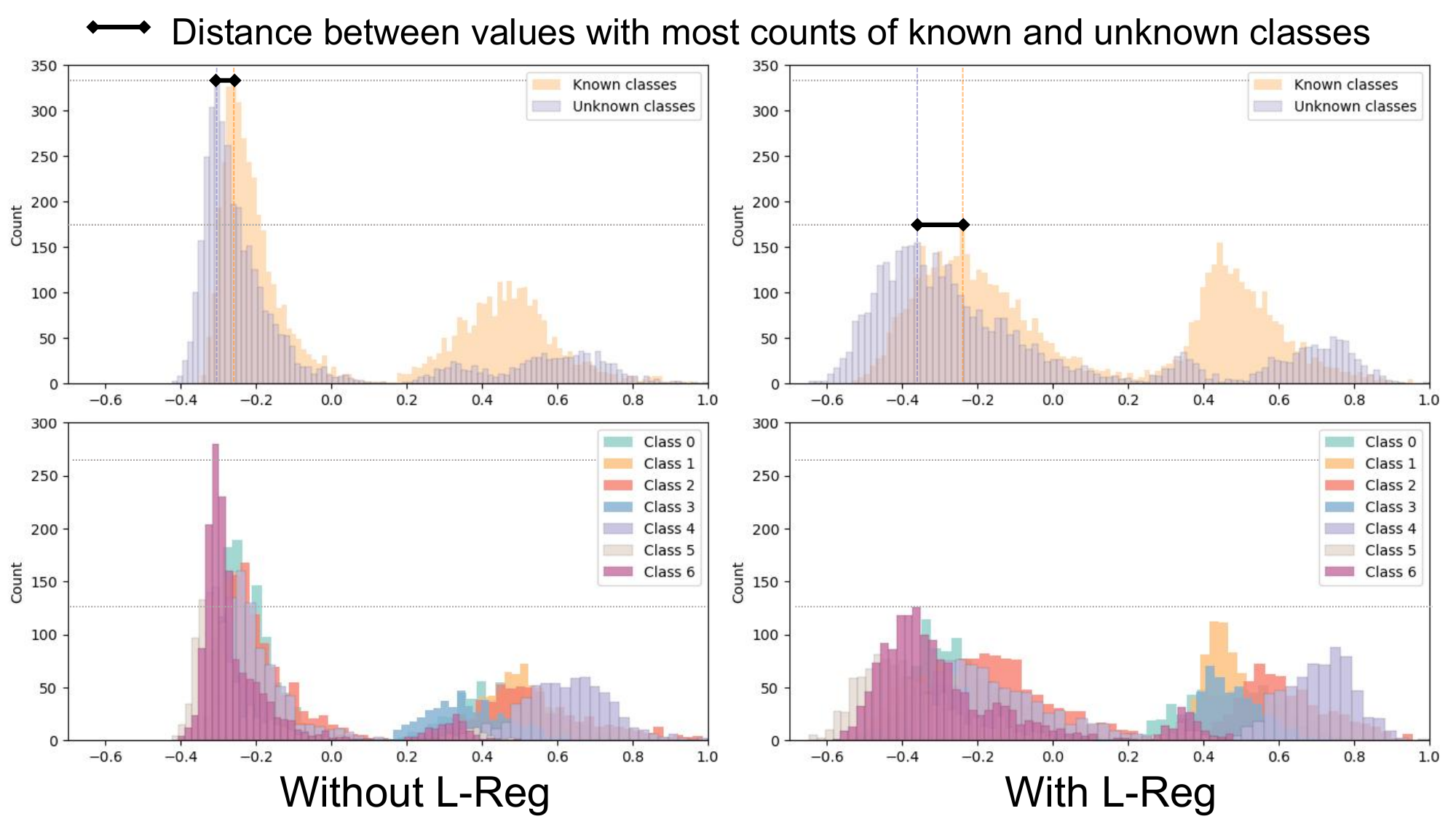}
    \caption{Visualizations of latent features form models trained using GMDG on PACS dataset without and with L-Reg under mDD+GCD setting using RegNetY-16G backbone, respectively.}
  \label{fig:feature}
\vspace{-0.5cm}
\end{wrapfigure}
\textbf{Balancing feature complexity:} 
L-Reg results in a more balanced distribution of features compared to the baseline, as illustrated in \cref{fig:feature}. This balanced distribution suggests the elimination of certain extracted semantics characterized by dominant frequencies across all samples.
Semantics that occur frequently across samples often lack decisiveness for classification. Hence, reducing their prominence contributes to more expressive feature space and less complex feature distributions. Coupled with the reduced classifier complexity, a simplified classifier achieved through L-Reg facilitates improved generalization across various settings.
{Specifically, the top row also indicates the distance between the feature distributions of the known and unknown classes, which is enlarged; thus, they are more dividable, leading to classification improvements.}



We present a logical-based theoretical analysis in \cref{sec:logic_framework} and provide the derivation details of L-Reg in \cref{sec:L-Reg_details}. In addition, we discuss the efficacy of L-Reg under various generalization settings in \cref{sec:L-Reg_for_settings}.
Furthermore, L-Reg serves as a plug-and-play loss function that is compatible with most existing frameworks. We conduct experiments applying L-Reg to various established approaches across different generalization settings, as outlined in \cref{sec:exp}.





\subsection{Logical 
framework for visual classification}
\label{sec:logic_framework}
This part provides the connections between logical reasoning and visual classification tasks.
We would like to remind readers of the framework for studying logics and link it with our practical scenarios.
\begin{definition} 
Following~\cite{andreka2017universal}, a logic $\mathcal{L}$ is defined as a five-tuple in the form: 
\begin{equation}
\label{eq:logic}
    \mathcal{L}=\left\langle 
    F_{\mathcal{L}},  
    M_{\mathcal{L}}, 
    \models_{\mathcal{L}},
    mng_{\mathcal{L}}, 
    \vdash_{\mathcal{L}}
    \right\rangle,
\end{equation}
where
     1) $F_{\mathcal{L}}$ denotes the set of formulas formed by images and labels ($X,Y$); 
     2) $M_{\mathcal{L}}$ represents different domains $D$ of $X$;
     3) $\models_{\mathcal{L}}$ is a binary relation relating the truth of whether the formulas are true or false, which has 
     $\models_{\mathcal{L}} \subseteq M_{\mathcal{L}} \times F_{\mathcal{L}}$;
     4)  $mng_{\mathcal{L}}: F_{\mathcal{L}} \times M_{\mathcal{L}} \longrightarrow \text { Sets }$  defines the meaning of  $X$ as determined by classifiers, where
     Sets indicate the class of all sets.   
    (5) $\vdash_{\mathcal{L}}$ symbolizes the provability relation of $\mathcal{L}$, evaluating formulas formed by $mng_{\mathcal{L}}$ is true or false in one possible world, such as the estimation criteria.  
    More details of $\mathcal{L}$ can be seen in
    \cref{app:logical_framework}.
\end{definition}

For clarity, we specify $ \mathcal{L}_{(X_s,Y_s)}=\left\langle 
F_{(X_s,Y_s)},  
D, 
\models_{(X_s,Y_s)},
h, 
\vdash_{(h(X),Y)}
\right\rangle$ 
as the logic formed on the given $X, Y$ sets.  With the goal for logic to generalize across a broader scenario and  provide extrapolation across all possible formulas in $\mathcal{L}$,
a good general logic $\mathcal{L}^*$ 
should be derived from $\mathcal{L}$ through the feature extractor $g$: 
\begin{equation}
    \label{eq:logic_cv}
    \mathcal{L}^*=\left\langle F_{(g(X_s),Y_s)}, D, \models_{(g(X_s), Y_s)},
    h, 
    \vdash_{(h \circ g(X), Y)}
    \right\rangle, s.t., \vdash_{(h \circ g(X), Y)} = \models_{(g(X_s), Y_s)}.
\end{equation}
Importantly, as a good general logic,  $F_{(g(X_s),Y_s)}$ and $h$ in $\mathcal{L}^*$ should form the \textit{atomic formulas}, i.e., the tuple of terms with a predicate: $ h\circ g(x) \text{ belongs/not belongs to class } y \text{ in domain } d \to Ture/False,\; \text{ where } x, y, d \in X,Y,D $, which makes that  $ \vdash_{(h \circ g(X_u), Y_u)} = \models_{(g(X_s), Y_s)}$ still holds. We simply denote one atomic formula in the form of $h(g(x), y, d)$ mapping to binary values.
{Additionally, 
$\vdash_{(h \circ g(X), Y)} = \models_{(g(X_s), Y_s)}$ in \cref{eq:logic_cv} can be safely omitted in the rest of the paper. 
Please see more details about the conditions of the good general logic in \cref{app:logical_framework}.
}

An additional tool is necessary to convert the logic problem into a continuous form, enabling the application of machine learning algorithms. The conditional entropy-based method enables a logically sound derivation of knowledge from the provided dataset with constraints \cite{rodder2000conditional}.
Specifically, the probabilistic inference process adheres to a probabilistic version of Modus Ponens: $A \to B, A \vdash B$ (if $A$ then $B$; not $A$ therefore not $B$). It is important to note that the logical propositions in probabilistic Modus Ponens are uncertain, with the conditional probability replacing the material implication $A\to B$.
This framework allows us to interpret logical deduction through the lens of entropy.
Therefore, for \cref{eq:logic_cv} which implies 
\begin{equation}
    \exists h\circ g, \; \forall (x,y) \in (X,Y) , \; \forall d \in D, \; h\circ g(x) \to y,
\end{equation}
finding $h\circ g$ through optimization  
is equivalent to 
\begin{equation}
\label{eq:general_obj}
        \max_{h, g} \mathbb{E}_{(x,y) \in (X,Y) ,  d\in D} P (y|g(x),d) - \mathcal{R} \Longleftrightarrow \min_{h, g} \mathbb{E}_{(x,y) \in (X,Y) ,  d\in D} H(y|g(x),d) + \mathcal{R},
\end{equation}
where $\mathcal{R}$ denotes any other possible regularization. 

As the logical framework for image classification takes shape, it becomes evident that the unresolved question of identifying an appropriate function $g$ to generate suitable atomic formulas emerges as a critical factor in ensuring the effectiveness of the overarching logic $\mathcal{L}^*$. 
This paper proposes L-Reg as the regularization to ensure 
$F_{(g(X_s),Y_s)}$ are formed by atomic formulas in \cref{sec:L-Reg_details}.

\subsection{Constructing atomic formulas using L-Reg}
\label{sec:L-Reg_details}
In this part, we show the derivation details of L-Reg 
the aims to ensure the formation of suitable atomic formulas, as depicted in  \cref{eq:logic_cv}.
As highlighted in~\cite{abbe2023generalization}, current algorithms may induce implicit biases towards unseen data, resulting in varied solutions for such data.
However, expecting an algorithm to generalize effectively to unseen data domains without appropriate incentivization, such as specifically designed regularization, is unreasonable.
Therefore, we aim to enhance the generalization capability of models by employing a logic-based regularization approach. To this end, we introduce the concept of \textit{semantic support} for image classification.

\begin{definition}[Semantic support]
\label{def:semantic_support}
    We denote $z = g(x)$, where $z\in Z$, as a set of compositions of these semantics: $z:= \{z^i\}_{i = 1}^M$, where $M$ is the number of dimensions or semantics. Notably, not all semantics in $z$ may be useful for deduction or inference. We define the subset $\gamma$ of $z$, extracted from the sample $x \sim \mathcal{X}$, as the semantic support of $x$ if $\gamma$ is sufficient for deducing the relationship between $x$ and a $y\sim \mathcal{Y}$.
\end{definition}


For instance, if the subset $\{z^1, z^2\} \subseteq z$ is sufficient for accurate inference, the values of other semantics $\{z^i\}_{i=3}^M$ will not impact the inference process. When $\{z^1, z^2\}$ constitutes the minimal combination of semantics required for inference, it is termed the semantic support.
We denote $\Gamma$ as the set of semantic supports of $X$ for deducing each individual class. 

\textbf{Derivation of L-Reg.}
Regarding \cref{eq:logic_cv}, if the semantic supports and their relationship with $Y$ form atomic formulas, \cref{eq:logic_cv} holds as a good general logic, and the generalization would be improved. 
Thus,
we aim to learn the latent features $Z$, which contain sufficient semantic supports for the deduction of $Y$: 
\begin{equation}
\label{eq:logic_c1}
    \exists \gamma\in \Gamma, \gamma \subseteq  z , \; \forall (z,y) \in (Z,Y),  \forall d \in D, \;  h(\gamma|d) \to y.
\end{equation}
Specifically,   
$g(\cdot)$ should meet the following:
\begin{equation}
\label{eq:logic_c2}
    \forall  (\Gamma_i,y_i), (\Gamma_j,y_j)  \in (Z,Y), 
    \forall d \in D,
    \; y_i \neq y_j \Longleftrightarrow \Gamma_i \neq \Gamma_j,
\end{equation}
i.e., the semantic support set for each class should be distinct. 
The multiple-class classification task has that $\forall \Gamma, |\Gamma| \leq M$. 
Under the constraints demonstrated in Eq.~(\ref{eq:logic_c1}) and Eq.~(\ref{eq:logic_c2}), we need to achieve the following through optimization:
\begin{equation}
\label{eq:reg_goal}
    \min_{h,g} H(Y|g(\Gamma), D), 
    \max_{h,g} H(Y|g(\Bar\Gamma), D) 
    \Longleftrightarrow  \min_{h,g} H(Y|g(\Gamma), D)- H(Y|g(\Bar\Gamma), D)
    ,
\end{equation}
where $\Bar\Gamma$ denotes the negation of $\Gamma$, i.e., the set of semantics which does not include semantic support. 


Intuitively, \cref{eq:reg_goal} regularizes that the model should be able to judge whether a sample belongs to a class by using a minimal set of semantic supports; simultaneously, the semantic support sets are also implicitly disentangled for each class, not only for maintaining rich and useful semantics but also for enhancing the independence of deduction of each class.
The actual collection of $\Gamma$ appears to be intractable during optimization. Hence, we resort to deriving its bounds. Regarding \cref{eq:reg_goal}, its former term can be elaborated as follows:
\begin{equation}
    H(Y|g(\Gamma), D) \le H(Y|h(z_i), D) \le   \mathbb{E}_{z_i \sim z} [H(Y|g(z_i), D)] ,
\end{equation}
where $z_i$ is minimal semantics form $z$,
and $\mathbb{E}_{i=1}^M H(Y|g(z_i), D) $ is the upper-bound for $\min_{h,g} H(Y|g(\Gamma), D)$. Therefore, minimizing $\mathbb{E}_{i=1}^M H(Y|g(z_i), D)$ is equivalent to minimizing $H(Y|g(\Gamma), D)$.
Meanwhile, for the latter in \cref{eq:reg_goal}, we have:
\begin{equation}
    H(Y|g(\Bar\Gamma), D ) \ge H(Y|g(z), D),
\end{equation}
where $H(Y|h(z)), D)$ is the lower-bound for $\max_{h,g} H(Y|g(\Bar\Gamma), D)$.
Combining the aforementioned bounds, we have the L-Reg objective as \cref{eq:final_reg_obj}.

\textbf{Interpretability of semantic supports roots in forming atomic formulas.}
The atomic formula $\mathcal{A}^y$ is of the form $h(g(x), y, d)$.
Our aim is to find the good (most) general $\mathcal{A}^{y*} \in \mathcal{A}^y$ for $y$ class from which the interpretability of L-Reg is derived.
Consider $\mathcal{A}^y_1, \mathcal{A}^y_2 \in \mathcal{A}^y$, if $\mathcal{A}^y_1$ is more general than $\mathcal{A}^y_2$, there will be a substitution $\psi$ such that $\mathcal{A}^y_1\psi = \mathcal{A}^y_2$ \cite{tsapara1998learning}. 
$\mathcal{A}^{y*}$  should meet $\mathcal{A}^{y*}\psi = \mathcal{A}^y_i\in \mathcal{A}^y$, which infers that $\gamma^y\psi = z^y$ (cf. \cref{eq:logic_c1}) for predication of $y$ where $\gamma^y$ is the {semantic support}. Note here that the form of $\mathcal{A}^y$ is constructed for $y \in Y$, i.e., predicate whether the sample belongs to the $y$ class. 
Considering multiple classes $y_i, y_j \in Y, i \neq j$, it has $\mathcal{A}^{y_i*} \neq \mathcal{A}^{y_j*}$ thus $\gamma^{y_i} \neq \gamma^{y_j}$ (cf. \cref{eq:logic_c2}), which constrains that different minimal semantic supports should be used for predicting different classes. 
The interpretability of L-Reg is based on $\mathcal{A}^{y*}$, compelling the model to use distinct minimal semantic supports for each class.  These minimal semantic supports can be interpreted as the most critical features for efficient prediction.
For example, as shown in \cref{fig:banner}, the model with L-Reg has learned the facial features of the person class (see more examples in Appendix~\cref{fig:dog,fig:elephant,fig:giraffe,fig:guitar,fig:horse,fig:person,fig:elephant,fig:person}), forming the (informal) atomic formula $h(\text{has a human face}, \text{is person}, d\in D) \rightarrow \text{True}$. 
Similarly, it also leads to $h(\text{not has a human face}, \text{is person}, d\in D) \rightarrow \text{False}$.



\section{L-Reg under different generalization settings}
\label{sec:L-Reg_for_settings}






\textbf{L-Reg under data-shift generalization.}
The task mDG endeavors to facilitate a model's ability to generalize to unseen domains by fostering invariance across seen domains~\cite{tan2024rethinking}. In the context of mDG, the term $|D| \ge 2$ in \cref{eq:general_obj} typically denotes multiple domains.
Traditionally, existing methods focus on minimizing domain gaps, leading to remarkable results~\cite{cha2022miro,tan2024rethinking}. However, it is noteworthy that even when the domain gap is effectively minimized, and $|D| = 1$ for the latent features can be considered, L-Reg still demonstrates its efficacy in promoting the generalization of $X_u$ from $D_u$.

\begin{proposition}[Effectiveness of L-Reg in enhancing data-shift generalization.] 
\label{prop:data} 
Assume the gap across all domains is well minimized. 
Let  $f^*$ denote the target model that generalizes to the data $X_u$ from the unseen domain with the lowest complexity.
For a model ${f}^{R}_{(X_s, Y_s)}, {f}_{(X_s, Y_s)}$ trained under the data-shift generalization setting (i.e., $(X_s, Y_s)$ 
is accessible and $\mathcal{Y}_s = \mathcal{Y}_u$). We have:
\begin{equation}
     GL({f}^{R}_{(X_s, Y_s)}, f^*, X_u) \le   GL({f}_{(X_s, Y_s)}, f^*, X_u).
\end{equation}
\end{proposition}
Please see proof details in \cref{app_prop:data}.
To illustrate \cref{prop:data}, consider the following intuitive example:
In the seen domains, all cats are either black or white, while all dogs are brown. Now, imagine encountering a sample labeled `a brown cat' from an unseen domain. Without the application of L-Reg, the model might erroneously classify it as a dog.
However, with L-Reg in place, the model is compelled to rely on minimal semantics for classification. This means filtering out irrelevant features such as color terms, thus enabling more accurate deductions.




\textbf{L-Reg under target-shift generalization.}
We demonstrate how L-Reg enhances generalized discovery in scenarios where only a subset of classes ($Y_s$) is available for training, and there may exist an overlap between the unseen classes ($Y_u$) and the seen classes ($Y_s$), denoted as $\mathcal{Y}_u \cap \mathcal{Y}_s \neq \emptyset$.
We define $Y_u / Y_s$ as the novel classes not included in $Y_s$, and $Y_u \sim Y_s$ as the seen classes for $X_u$ classification, where $|D| = 1$.
Building upon \cref{prop:data}, L-Reg further enhances GCG by improving the generalization performance on $Y_u$.

\begin{proposition} [L-Reg improves target-shift generalization]
\label{prop:target}
When $|D| = 1$, L-Reg promotes generalization performance on $Y_u$ under the target-shift scenario.
\end{proposition}

\begin{proof}
    When $|D| = 1$, 
    since all $Y$ belongs to a close set, minimizing $ - H(Y_s|g(\Bar\Gamma), D)$ is equivalent to the following:
    \begin{equation}
        \min_{h,g} - H(Y_s|g(\Bar\Gamma))
        \Longleftrightarrow \min_{h,g} H(\Bar{Y_s}|g(\Bar\Gamma)),
    \end{equation}
    where $\Bar{Y_s}$ is the negation of $Y_s$, i.e., $Y_u / Y_s$. 
    In this situation, if one sample does not contain sufficient semantic support to be classified under $Y_s$, it otherwise will be assigned under $Y / Y_s$, promoting performance for both  $Y_u / Y_s$  and  $Y_u \sim Y_s$.
    Therefore, the generalization performance on the unseen classes will be improved by L-Reg.  
\end{proof}

\textbf{L-Reg under all-shift generalization.}
When the domain gap is sufficiently minimized and $|D| = 1$ can be considered, the combination of \cref{prop:data} and \cref{prop:target} demonstrates that L-Reg enhances generalization performance on both novel classes ($Y_u / Y_s$) and seen classes ($Y_u \sim Y_s$) for $X_u$ from other domains.
Our experiments validate that L-Reg, when applied in scenarios with well-minimized domain gaps, consistently improves generalization across all shifts.

\begin{table}[t]
\centering
\caption{MDG results: Comparison between the proposed and previous non-ensemble and ensemble mDG methods. The best results for each group are highlighted in \textbf{bold}.
Improvement and degradation in our approach from GMDG are highlighted in \red{red}.
}
\label{tab:mDG_results}
\resizebox{0.8\linewidth}{!}{%
\begin{tabular}{llllll|l}
 \toprule
Test domain & \bf PACS & \bf VLCS & \bf OfficeHome & \bf TerraIncognita & \bf DomainNet & \bf Avg. \\ \midrule
MMD~\citep{li2018domainMMD} & 84.7±0.5 & 77.5±0.9 & 66.3±0.1 & 42.2±1.6 & 23.4±9.5 & 58.8 \\
Mixstyle~\citep{zhou2021domain} & 85.2±0.3 & 77.9±0.5 & 60.4±0.3 & 44.0±0.7 & 34.0±0.1 & 60.3 \\
GroupDRO~\citep{sagawa2019distributionally} & 84.4±0.8 & 76.7±0.6 & 66.0±0.7 & 43.2±1.1 & 33.3±0.2 & 60.7 \\
IRM~\cite{arjovsky2019invariant} & 83.5±0.8 & 78.5±0.5 & 64.3±2.2 & 47.6±0.8 & 33.9±2.8 & 61.6 \\
ARM~\citep{zhang2021adaptive} & 85.1±0.4 & 77.6±0.3 & 64.8±0.3 & 45.5±0.3 & 35.5±0.2 & 61.7 \\
VREx~\citep{krueger2021out} & 84.9±0.6 & 78.3±0.2 & 66.4±0.6 & 46.4±0.6 & 33.6±2.9 & 61.9 \\
CDANN~\citep{li2018deep} & 82.6±0.9 & 77.5±0.1 & 65.8±1.3 & 45.8±1.6 & 38.3±0.3 & 62.0 \\
DANN~\citep{ganin2016domain} & 83.6±0.4 & 78.6±0.4 & 65.9±0.6 & 46.7±0.5 & 38.3±0.1 & 62.6 \\
RSC~\citep{huang2020self} & 85.2±0.9 & 77.1±0.5 & 65.5±0.9 & 46.6±1.0 & 38.9±0.5 & 62.7 \\
MTL~\citep{blanchard2021domain} & 84.6±0.5 & 77.2±0.4 & 66.4±0.5 & 45.6±1.2 & 40.6±0.1 & 62.9 \\
MLDG~\citep{li2018learning} & 84.9±1.0 & 77.2±0.4 & 66.8±0.6 & 47.7±0.9 & 41.2±0.1 & 63.6 \\
Fish~\citep{shi2021gradient} & 85.5±0.3 & 77.8±0.3 & 68.6±0.4 & 45.1±1.3 & 42.7±0.2 & 63.9 \\
ERM~\citep{Vapnik1998ERM} & 84.2±0.1 & 77.3±0.1 & 67.6±0.2 & 47.8±0.6 & 44.0±0.1 & 64.2 \\
SagNet~\citep{nam2021reducing} & {86.3}±0.2 & 77.8±0.5 & 68.1±0.1 & 48.6±1.0 & 40.3±0.1 & 64.2 \\
SelfReg~\citep{kim2021selfreg} & 85.6±0.4 & 77.8±0.9 & 67.9±0.7 & 47.0±0.3 & 42.8±0.0 & 64.2 \\
CORAL~\citep{sun2016deep} & 86.2±0.3 & 78.8±0.6 & 68.7±0.3 & 47.6±1.0 & 41.5±0.1 & 64.5 \\
mDSDI~\cite{bui2021exploiting} & 86.2±0.2 & 79.0±0.3 & 69.2±0.4 & 48.1±1.4 & 42.8±0.1 & 65.1 \\ 
\midrule
 & \multicolumn{6}{c}{Use RegNetY-16GF~\citep{singh2022revisiting} as oracle model.} \\ 
MIRO~\cite{cha2022miro} (ECCV23) & {97.4}±0.2 & 79.9±0.6 & 80.4±0.2 & 58.9±1.3 & 53.8±0.1 & 74.1 \\
{GMDG}~\cite{tan2024rethinking} (CVPR24) & 97.3±0.1 & {82.4}±0.6 & {80.8}±0.6 & {60.7}±1.8 & {54.6}±0.1 & {75.1} \\
\rowcolor{mygray}\textbf{GMDG + L-Reg} & \textbf{97.4}±0.2\red{$^{0.1\uparrow}$}  &  \textbf{82.4}±0.0\red{$^{0.1\uparrow}$} &  \textbf{80.9}±0.5\red{$^{0.1\uparrow}$} & \textbf{62.9}±0.9\red{$^{2.2\uparrow}$} & \textbf{55.3}±0.0\red{$^{0.8\uparrow}$}  & \textbf{75.8}\red{$^{0.7\uparrow}$} \\
\bottomrule 
\end{tabular}%
}
 \vspace{-.5cm}
\end{table}

\section{Experiments}
\label{sec:exp}

To validate L-Reg, three groups of experiments under the three kinds of settings are conducted. Notably, all baselines we used already incorporate the $L_2$ regulation in the form of weight decay. 
{We also compare other commonly used regularization terms, such as independence or sparsity regularization on $Z$. More results in
\cref{app:compare_to_more}indicate that our L-Reg also surpasses them.}



\subsection{Experiments on mDG}

\textbf{Experimental settings.}
We operate on the DomainBed suite~\cite {gulrajani2020search} and leverage standard leave-one-out cross-validation as the evaluation protocol.
We test L-Reg with GMDG~\cite{tan2024rethinking} on $5$ real-world benchmark datasets: PACS~\cite{li2017deeper}, VLCS~\cite{fang2013unbiased}, OfficeHome~\cite{venkateswara2017deep}, TerraIncognita~\cite{beery2018recognition}, and DomainNet~\cite{peng2019moment}.
Following MIRO~\citep{cha2022miro} and GMDG~\cite{tan2024rethinking}, the RegNetY-16GF backbone with SWAG pre-training~\cite{singh2022revisiting}) is used.
Specifically, we train the backbone using GMDG with L-Reg.
Accuracy is adopted as the evaluation metric, and the results of the averages from three trials of each experiment, with standard deviations, are presented. 
See Supplementary~\ref{app:experiments} for more experimental details.

\textbf{Results.}
The experimental results presented in \cref{tab:mDG_results} demonstrate the efficacy of L-Reg in improving the performance of GMDG across all datasets in mDG classification tasks. Notably, more substantial improvements are observed when the GMDG baseline achieves relatively low accuracy. These observed enhancements provide empirical support for \cref{prop:data}. 
Please see using L-Reg with basic ERM in \cref{app:compare_to_erm}.
For detailed insights into each domain within each dataset, please refer to \cref{app:mDG}.

\subsection{Experiments on GCD}
 
\begin{wraptable}{r}{0.4\linewidth}
\vspace{-1.5cm}
\caption{GCD results: Average results across all datasets of PIM with L-Reg. Improvements and degradation are highlighted in \red{red} and \blue{blue}, respectively.}
\label{tab:GCD_res}
\resizebox{\linewidth}{!}{%
\begin{tabular}{l|lll}
\toprule
\multicolumn{1}{c|}{Average}  & {All}  & {Known}  & {Unknown}  \\
\midrule
K-means \cite{macqueen1967classification}               & 44.7          & 46.0          & 43.9          \\
RankStats+ \cite{han2021autonovel} (TPAMI-21) & 38.6          & 54.6          & 25.6          \\
UNO+ ~\cite{fini2021unified} (ICCV-21)        & 51.2          & 74.5          & 36.7          \\
ORCA \cite{cao2022openworld} (ICLR-22)        & 46.3          & 51.3          & 41.2          \\
ORCA - ViTB16         & 56.7          & 65.6          & 49.9          \\
GCD \cite{vaze2022generalized} (CVPR-22)         & 60.4          & 71.8          & 52.9          \\
RIM \cite{krause2010discriminative} (NeurIPS-10)      & 62.0          & 72.5          & 55.4          \\
TIM \cite{boudiaf2020information} (NeurIPS-20)      & 62.7          & 72.6          & 56.4          \\
\midrule
PIM \cite{chiaroni2023parametric} (ICCV-23)         & 67.4          & \textbf{79.3}          & 59.9          \\
\rowcolor{mygray}  \textbf{PIM + L-Reg}  & \textbf{68.8}\red{$^{1.4\uparrow}$} & {79.0}\blue{$^{0.3\downarrow}$} & \textbf{62.7}\red{$^{2.8\uparrow}$}\\
\bottomrule
\end{tabular}%
}

\vspace{0.3cm}
\caption{MDG+GCD results: Averaged accuracy scores for all, known and unknown classes across all five datasets. Improvements and degradation are highlighted in \red{red} and \blue{blue} respectively.
}
\label{tab:MDG+GCD_res}
\resizebox{\linewidth}{!}{%
\begin{tabular}{cl|ccc}
\toprule
Method         &Domain gap& All                         & Known                         & Unknown                          \\ \hline
ERM             &Not& 44.69                       & 59.33                       & 23.54                        \\
\textbf{+L-Reg} &minimized& 45.50                       & 61.43                       & 21.63                        \\
\rowcolor{mygray}Imp.   && \red{0.81} & \red{2.09} & \blue{-1.91} \\ \hline
PIM            &Not& 46.95                       & 60.35                       & 26.90                        \\
\textbf{+L-Reg} &minimized& 47.27                       & 60.83                       & 26.34                        \\
\rowcolor{mygray}Imp.  && \red{0.32} & \red{0.48} & \blue{-0.57} \\ \hline
MIRO           &Not sufficiently& 49.67                       & 68.86                       & 25.79                        \\
\textbf{+L-Reg }         &minimized& 52.11                       & 71.26                       & 26.49                        \\
\rowcolor{mygray}Imp.   && \red{2.44} & \red{2.39} & \red{0.71}  \\ \hline
GMDG           &Sufficiently& 47.94& 	68.75& 	20.68             \\
\textbf{+L-Reg} &minimized&51.94	&69.87&	27.68                \\
\rowcolor{mygray}Imp.   && \red{4.00} & \red{1.12} & \red{7.01} \\
\bottomrule
\end{tabular}%
}
\vspace{-.5cm}
\end{wraptable}

\textbf{Experimental settings.}
We validate our approach through training PIM additionally with L-Reg.
Six image datasets are adopted to validate the feasibility of our proposed RPIM compared to other competitors, including three generic object recognition datasets, CIFAR10~\cite{krizhevsky2009learning},
CIFAR100~\cite{krizhevsky2009learning} 
and 
ImageNet-100~\cite{deng2009imagenet};
two fine-grained datasets CUB~\cite{wah2011caltech} and Stanford Cars~\cite{krause20133d}; and the long-tail dataset Herbarium19~\cite{tan2019herbarium}.
Following prior works \cite{vaze2022generalized,chiaroni2023parametric}, we use the proposed accuracy metric from \cite{vaze2022generalized} of all classes, known classes, and unknown classes for evaluation. 
Please see a detailed description of the experimental setup in \cref{app:GCD}.

\textbf{Results.}
The average results across all datasets for utilizing L-Reg with PIM are presented in \cref{tab:GCD_res}, while detailed dataset-specific information is available in \cref{app:experiments}~\cref{tab:GCD_res_details}. The results highlight that L-Reg consistently increases the accuracy of all unknown classes across all datasets, thus confirming the validity of \cref{prop:target}.
However, it is notable that L-Reg may marginally compromise the performance of known classes, as it reduces the size of semantic support for deducing $Y$, thereby reducing the information available for known classification. Nevertheless, this compromise is deemed acceptable given the significant improvements observed for the unknown classes.



\subsection{Experiments on mDG + GCD}

\textbf{Experimental settings.}
We utilize datasets designed for mDG tasks to conduct mDG + GCD experiments. During the training stage, only samples from seen domains are available, with half of the classes masked as unknown, and only their unlabeled data are utilized. 
Notably, even though all the unlabeled data originates from unknown classes during training, this prior knowledge is not assumed or constrained, aligning the setting with GCD.
Similar to mDG, we adopt the leave-one-out cross-validation method. This entails testing each domain in each dataset as the unseen domain.  The performance is tested on unseen domains by employing GCD metrics. 
To validate L-Reg's efficacy comprehensively, we re-implement four methods under the mDG + GCD setting, testing them both with and without L-Reg. The four methods include ERM, PIM, MIRO, and GMDG. ERM serves as the baseline approach without additional regularization, while PIM maximizes information without minimizing domain gaps. MIRO and GMDG focus on minimizing domain gaps, with GMDG offering a comprehensive approach in this regard.
It is worth noting that PIM has been re-implemented. For further experimental details, please refer to \cref{app:mDG+GCD}.

\textbf{Results.}
The averaged results across all unseen domains of all datasets are summarized in \cref{tab:MDG+GCD_res}. For a detailed breakdown of results for each domain in each dataset, please refer to \cref{app:mDG+GCD}. As discussed in \cref{prop:data} and \cref{prop:target},
a noticeable trend is observed wherein, as the domain gap is gradually minimized, the improvements for unknown classes increase, with the best results achieved using GMDG with L-Reg. 

\begin{wrapfigure}{r}{0.45\linewidth}
\vspace{-1cm}
  \centering
  \includegraphics[width=\linewidth]{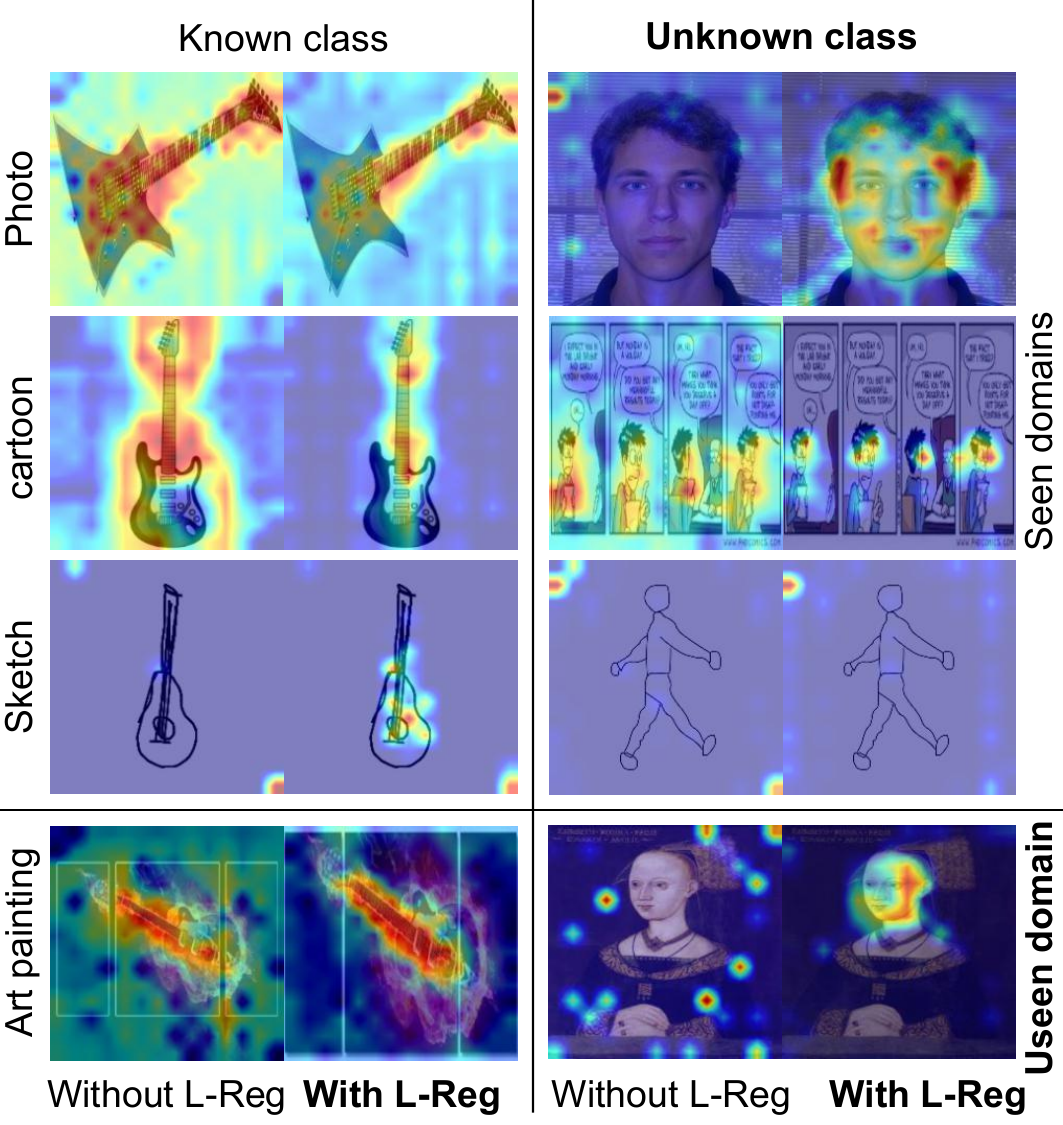}
  \captionof{figure}{GradCAM visualizations of GMDG trained without and with L-Reg. The seen, unseen domains and known, unknown classes are denoted. 
  }
  \label{fig:gmdg_res_vis}
\vspace{-.4cm}
\end{wrapfigure}

\textbf{L-Reg forms atomic formulas and improves interpretability.}
Furthermore, \cref{fig:gmdg_res_vis} provides visual insights into the behavior of models trained with L-Reg. Evidently, these models tend to focus on minimal semantics sufficient for class distinctions.
{For the known classes, the efficacy of L-Reg can be intuitively understood as extracting the minimal semantic supports for a given class label.
For instance, the presence of a guitar's fingerboard, even in unseen domains, helps classify a sample as belonging to the guitar category, whose informal forms can be denoted as $h(\text{has fingerboard}, \text{is guitar}, d \in D) \to\text{True}$ and $h(\text{not has fingerboard}, \text{is guitar}, d \in D) \to\text{False}$. 
For all known classes, samples with these minimal semantic supports are recognized accordingly.
In contrast, if a sample lacks these minimal supports for any known class, it is very likely categorized as an unknown class. This behavior stems from Paper Eq.10 which ensures $\mathcal{A}^{y_i*} \neq \mathcal{A}^{y_j*}$ through constraining $\gamma^{y_i} \neq \gamma^{y_j}$. 
L-Reg further enhances the model's ability to identify minimal supports for unknown classes by filtering out co-covariant features associated with other classes and thus generalizing to unseen domains.
Therefore, the very interpretable features for unknown classes from unseen domains can be extracted using L-Reg. \cref{fig:gmdg_res_vis} (right side) demonstrates that the model with L-Reg can even extract facial features for the unknown person class and can generalize this to the unseen domain. Similarly, here we obtain (informal) atomic formulas as $h(\text{has a face}, \text{is person}, d \in D)\to\text{True}$,  $h(\text{not has a face}, \text{is person}, d \in D)\to\text{False}$.
}

{However, as shown in Row 3, significant domain shifts, such as those between the sketch domain and other domains, pose challenges. 
Specifically, the differences between the stick-figure style of sketches of persons and figures from other domains can hinder the model's ability to cluster sketches with other domains' figures when the class label is unknown. 
Thus, under this circumstance, the model may fail to extract meaningful features from those sketches.
We acknowledge this limitation and will explore solutions in future work.
}

\begin{wraptable}{r}{0.45\linewidth}
\vspace{-.8cm}
\captionof{table}{Averaged results of applying L-Reg to different layers across domains in PACS.}
\label{tab:different_layer}
\resizebox{\linewidth}{!}{%
\begin{tabular}{l|ccc}
\toprule
& \multicolumn{1}{l}{All} & \multicolumn{1}{l}{Known} & \multicolumn{1}{l}{Unkown} \\
\midrule
GMDG                                                                            & 58.33                   & 91.46                     & 10.18                      \\
\rowcolor{mygray}\begin{tabular}[c]{@{}l@{}}L-Reg: Deep layer \end{tabular}             & \textbf{67.82}                   & \textbf{91.86}                     & 31.33                      \\
\rowcolor{mygray}\begin{tabular}[c]{@{}l@{}}L-Reg: Earlier and the deep layers\end{tabular} & 58.97                   & 80.73                     & \textbf{35.05}     \\
\bottomrule
\end{tabular}%
}
\vspace{-0.2cm}
\end{wraptable}
\textbf{L-Reg should be applied to features from deep layers.}
One crucial precondition highlighted in the theoretical analysis is that L-Reg operates effectively with a representation $Z$, where each dimension represents independent semantics.
The semantic features usually come from the deeper layers of the model architecture~\cite{tan2024semantic}. However, 
\cref{tab:different_layer} shows that applying L-Reg to features from earlier layers, which may not necessarily represent semantics, leads to a degradation in performance for known classes, albeit improving performance for unknown classes. This phenomenon arises due to the potential interdependence among features from earlier layers, resulting in penalization that may hinder the capture of semantic supports essential for known classes.
To ensure generalization improvements without significant compromise to the performance of known classes, we advocate for applying L-Reg specifically to features extracted from deeper layers, such as the bottleneck layer. These suggest that the compromised results observed in \cref{tab:GCD_res} could be attributed to the less depth of the model structure, which fails to provide the expected semantic features.

\subsection{Apply L-Reg to congestion prediction for circuit design.}

\textbf{Experimental settings.}
We also test L-Reg in Congestion prediction on the CircuitNet \cite{chai2023circuitnet} dataset by using CircuitFormer \cite{zou2024circuit} backbone. 
The congestion prediction is for circuit design and benefits from logical reasoning-based approaches.
All parameters, except for L-Reg, remain consistent with CircuitFormer, and we follow its metrics.

\begin{wraptable}{r}{0.5\linewidth}
    \centering
    \centering
    \captionof{table}{\textbf{Results of Congestion prediction:} Congestion prediction is proposed for circuit design.}
    \label{Tab:circuitNet_res}
    \resizebox{\linewidth}{!}{%
    \begin{tabular}{c|c c c}
    \hline
       & pearson & spearman & kendall \\
        \hline
      Gpdl with UNet++ & 0.6085 & 0.5202 & 0.3855 \\
      CircuitFormer (SOTA) & 0.6374 & 0.5282 & 0.3935\\
      \rowcolor{mygray}\textbf{\begin{tabular}[c]{@{}l@{}}CircuitFormer + L-Reg (\textbf{Ours}) \end{tabular}} &  \textbf{0.6553}  & \textbf{0.5289}  & \textbf{0.3944} \\
      \hline
    \end{tabular} %
    }
\vspace{-0.3cm}
\end{wraptable}

\textbf{Results.}
\cref{Tab:circuitNet_res} shows the results of prediction results on the CircuitNet dataset. 
We also include the results of  Gpdl with UNet++ and CircuitFormer for better comparison. 
Notably, the improvements brought by CircuitFormer with L-Reg across all metrics, especially for the pearson metric can be observed. 
The consistent improvement with L-Reg across all metrics indicates L-Reg's feasibility. 

\section{Related work}
\textbf{Logical reasoning for deep learning.} 
Current studies focus on length generalization or symbolic reasoning in the logic-based scope.
For length generalization, 
\cite{abbe2023generalization} proposes the generalization to the unseen setting, theoretically verifying that commonly used models can generalize to the unseen and degree curriculum promotes the generalization ability of the transformer, followed by~\cite{ahuja2024provable,abbe2024provable,xiao2024theory}. 
Another branch is to improve the logical reasoning ability for abstract symbols, such as learning the logical-based temples and expecting the model to generalize to unseen samples~\cite{boix2023can,li2024neuro}. These studies are closely related to languages, such as generating longer answering sequences or solving mathematical problems in large language models, lacking explicit connections to visual tasks.
\cite{barbiero2022entropy} 
delves into the logical explanations in image classification by explicitly extracting logical relationships. While this logical-based approach sheds light on the interpretability of image classification models, its specific benefits for visual generalization remain relatively unexplored.


\textbf{Multi-domain generalization.} 
Current approaches for mDG in image classification focus on learning invariant representation across domains. 
Previous approaches like DANN~\cite{ganin2016domain} minimize feature divergences between source domains. CDANN~\cite{li2018deep}, CIDG~\cite{li2018domain}, and MDA~\cite{hu2020domain} consider conditions for learning conditionally invariant features. MIRO~\cite{cha2022miro} and GMDG~\cite{tan2024rethinking} take advantage of pre-trained models to improve generalization. 
Specifically, in comparison to 
MIRO, GMDG proposes a general entropy-based learning objective for mDG and sufficiently minimizes the domain gaps, yielding better generalization results.


\textbf{Generalized category discovery.}
Generalized category discovery, pioneered by~\cite{vaze2022generalized}, addresses unlabeled samples with both known and unknown classes. 
Furthermore, PIM~\cite{chiaroni2023parametric} integrates InfoMax into generalized category discovery, effectively handling imbalanced datasets and surpassing GCD on both short- and long-tailed datasets.

\section{Conclusion}

This paper presents L-Reg, a logical regularization approach tailored for image classification tasks using logic analysis frameworks.
L-Reg yields better generalization across different settings by fostering balanced feature distributions and streamlining the classification model's complexity. Rigorous theoretical analyses and empirical validations underscore its efficacy, as L-reg consistently improves generalization performance with different frameworks under various scenarios.

\textbf{Limitation.}  L-Reg narrows the extent of semantic supports, potentially diminishing the amount of information available for classification and leading to certain trade-offs in the performance of seen datasets. 
This effect is evidenced by the slight decline in the accuracy of known classes when L-Reg is applied, as shown in \cref{tab:GCD_res}. A similar phenomenon is observed in \cref{fig:gmdg_res_vis}, where the model fails to recognize a person in the sketch domain lacking facial features. 
Analysis from \cref{tab:different_layer} suggests that these compromises may result from improper $Z$. Future work should focus on mitigating potential compromises on seen datasets by exploring strategies for better capturing $Z$ through improved model architecture design. We offer more experimental results of possible solutions to this limitation in \cref{app:limiation}, such as further constraining the independence of each dimension in $Z$. Those results may suggest a direction for future work.

\section*{Acknowledgments}
{
The work was partially supported by the following: National Natural Science Foundation of China under No. 92370119, No. 62376113, No. 62206225, and No. 62276258.
}

{
\small

\bibliographystyle{plain}
\bibliography{main.bib}

}

\newpage
\appendix

\section{Broader impact}
Our regularization term based on logic for image classification offers significant potential beyond academia. By integrating logical constraints, our approach enhances model robustness, interpretability, and ethical alignment. This translates into improved performance on real-world tasks such as disease diagnosis in healthcare and mitigating biases in decision-making systems. Our work fosters interdisciplinary collaboration and contributes to the responsible deployment of AI technologies, ultimately benefiting society through enhanced efficiency, fairness, and transparency in machine learning applications.

\section{Details of the logical framework for visual classification task}
\label{app:logical_framework}
We provide more details of the connections between logical reasoning and visual classification tasks.
\begin{definition} 
\label{def:logic}
Following~\cite{andreka2017universal}, a logic $\mathcal{L}$ is a five-tuple defined in the form:
\begin{equation}
\label{eq:logic}
    \mathcal{L}=\left\langle 
    F_{\mathcal{L}},  
    M_{\mathcal{L}}, 
    \models_{\mathcal{L}},
    mng_{\mathcal{L}}, 
    \vdash_{\mathcal{L}}
    \right\rangle,
\end{equation}
where
\begin{itemize}
    \item $F_{\mathcal{L}}$ is a set of all formulas of $\mathcal{L}$.  $F_{\mathcal{L}}$ arbitrarily refers to any collections that can be `expressed' by language $\mathcal{L}$. 
     Therefore,  $F_{\mathcal{L}}$ could be not only a collection of languages but also images and labels ($X,Y$) for computer vision cases.
    \item $M_{\mathcal{L}}$ is a class called the class of all models (or possible worlds) of $\mathcal{L}$; intuitively, this can be considered as different domains $D$ of $X$.
     \item $\models_{\mathcal{L}}$ is a binary relation, $\models_{\mathcal{L}} \subseteq M_{\mathcal{L}} \times F_{\mathcal{L}}$, called the validity relation of $\mathcal{L}$. For example, in the known set, the ground truth label of the image is given as truth, which is the validity relation.  
     \item $mng_{\mathcal{L}}: F_{\mathcal{L}} \times M_{\mathcal{L}} \longrightarrow \text { Sets }$
    where Sets is the class of all sets.      
    $mng_{\mathcal{L}}$ is a function with domain $F_{\mathcal{L}} \times M_{\mathcal{L}}$, called the meaning function of $\mathcal{L}$: Intuitively, $mng_{\mathcal{L}}$ extracts the meaning of the expressions can be understood as the classifiers.
    \item $\vdash_{\mathcal{L}}$ represents the provability relation of $\mathcal{L}$, telling us which formulas are `true' in which possible world and usually is definable from  $mng_{\mathcal{L}}$, such as the estimation criteria in the machine learning system.  
\end{itemize}
\end{definition}

Accordingly and still following~\cite{andreka2017universal}, a good general logic is defined as:
\begin{definition}[General logic]: A general logic is a class:
\begin{equation}
   \mathcal{L}^* := \left\langle  \mathcal{L}^P:P\in Sig \right\rangle,
\end{equation}
    where $Sig$ is a class of sets; $ \mathcal{L}^P=\left\langle 
    F_{\mathcal{L}}^P,  
    M_{\mathcal{L}}^P, 
    \models_{\mathcal{L}}^P,
    mng_{\mathcal{L}}^P, 
    \vdash_{\mathcal{L}}^P, 
    \right\rangle$ is a compositional
    logic in the sense of \cref{def:logic} for $P\in Sig$, and for any sets $P, Q\in Sig$ satisfies the
    following conditions:
    \begin{enumerate}
        \item $P$ is the set of atomic formulas of $\mathcal{L}^P$.
        \item $Cn(\mathcal{L}^P ) = Cn(\mathcal{L}^Q):= Cn(\mathcal{L}^*)$ where $Cn(\cdot)$ is called the set of logical connectives of the given logic (these are operation symbols with finite or infinite ranks).
        \item  Any bijection $f : P → Q$ that extends to a bijection between the tautological formula algebras of $\mathcal{L}^P$ and $\mathcal{L}^Q$ induces an isomorphism
        between $\mathcal{L}^P$ and $\mathcal{L}^Q$.
        \item If $P \subseteq Q$, then $\mathcal{L}^P$ is a sublogic of $\mathcal{L}^Q$.
        \item  For any $P \in Sig$ and set $H$, there is a $P' \in Sig$ such that $P'$ is disjoint from $H$ and $\mathcal{L}^{P'}$ is an isomorphic copy of $\mathcal{L}^P$.
        \item  The union of a system $P_i, i \in {I} $ of pairwise disjoint sets $P_i$ from
        $Sig$ belongs to $Sig$, whenever $I$ is not empty.  Let $\mathfrak{F r}(\cdot)$ denotes free algebra,  $\operatorname{Alg}_m(\mathcal{L})$ represents $\left\{m n g_{\mathfrak{M}}(\mathfrak{F}): \mathfrak{M} \in M\right\}$ where $\mathfrak{F}$ denotes the term algebra.
        Further, the tautological
        congruence of the logic belonging to the disjoint union $P$ is generated
        in $\mathfrak{F r}\left(\operatorname{Alg}_m\left(\mathcal{L}^P\right), P\right)$ as a congruence by the union of the tautological
        congruence relations of the logics belonging to  $P_i, i \in {I} $. 
        \item $Sig$ contains at least one non-empty set.
    \end{enumerate}
\end{definition}
Our L-Reg aims to regularize the semantics extracted by $g$ and the classifier to satisfy condition~$1$.

\textbf{$\vdash_{(h \circ g(X), Y)} = \models_{(g(X_s), Y_s)}$ in \cref{eq:logic_cv} can be safely omitted in the rest of the paper. }
{Consider the logic formed on $X,Y$: $\mathcal{L}_{(X_s,Y_s)}=\left\langle 
F_{(X_s,Y_s)},  
D, 
\models_{(X_s,Y_s)},
h, 
\vdash_{(h(X),Y)}
\right\rangle$. 
Assume we want to study the logic of $\vdash$ which can be defined in the form of $\mathcal{L}_{\vdash}:
\stackrel{\text{def}}{=} \left\langle  
F_{X_s,Y_s}, D_{\vdash}, h_{\vdash}, 
\models_{\vdash}
\right\rangle 
$, where $D_{\vdash}, h_{\vdash}, 
\models_{\vdash}$ are pseudo-components associated with $\vdash$. 
Particularly, $D_{\vdash}$ is a subset of all possible world/domains from $F_{(X_s,Y_s)}$: $D_{\vdash} \stackrel{\text{def}}{=}\{T \subseteq F_{(X_s,Y_s)}: T$ is closed under $\vdash_{(h(X),Y)}\}$.
For any $T\in D_{\vdash}$ and $a \in F_{(X_s,Y_s)}$, it has
$h_{\vdash} (a, T)\stackrel{\text{def}}{=}\{ b \in F: T \vdash (a \leftrightarrow b)
\}$. Further, $\models_{\vdash}$ in $T \in D_{\vdash}$ is defined as $T\models_{\vdash} a  
\stackrel{\text{def}}{\Leftrightarrow} a \in T
$. 
\cite{andreka2017universal} points out that the following condition is almost always satisfied: (Cond) $\forall a, b \in F_{\vdash}, d \in D_{\vdash}$, we have $(h_{\vdash}(a,d) = h_{\vdash}(b, d)) \text{ and } d \models_{\vdash} a \Rightarrow d \models_{\vdash} b$. 
Therefore, the semantical consequence relation induced by $\models_{\vdash}$ coincides with the original syntactical $\vdash_{(h \circ g(X), Y)}$ while Cond holds. Due to that $D_{\vdash} \subseteq D$, it infers that $\models_{(g(X_s), Y_s)}$ coincides with $\models_{\vdash}$. Therefore, $\vdash_{(h \circ g(X), Y)} = \models_{(g(X_s), Y_s)}$ can be safely omitted in the rest of the paper. }

\section{Details of proofs}
\label{app:proofs}

\begin{proposition}[L-Reg reduces the complexity of the model, promoting data-shift generalization performance.]
\label{app_prop:data} 
Assume the domain gap is well minimized. Consider a $f^*$ is the target model that generalizes to the unseen with the lowest complexity.
There are ${f}^{R}_{(X_s, Y_s)}, {f}_{(X_s, Y_s)}$ trained under the setting of data-shift generalization (i.e., $(X_s, Y_s)$ is accessible and $\mathcal{Y}_s = \mathcal{Y}_u$), it has that:
\begin{equation}
     GL({f}^{R}_{(X_s, Y_s)}, f^*, X_u) \le   GL({f}_{(X_s, Y_s)}, f^*, X_u),
\end{equation}
\end{proposition}

\begin{proof}
    We assume the loss is achieved for the tractable form by minimizing the mean squared error. In that case, 
    we have ${f}^{*}_{(X_s, Y_s)}$ for the given training set as: 
    \begin{equation}
        {f}^{*}_{(X, Y)} = 
        (g (X)^T g (X))^{-1} h\circ g (X)^T Y_s = (Z^T Z)^{-1} h (Z^T) Y,
    \end{equation}
    In comparison to ${f}^{*}$, ${f}_{(X_s, Y_s)}$  for the given seen sets is as:
    \begin{equation}
        {f}_{(X_s, Y_s)} = (Z_s^T Z_s)^{-1} h (Z_s^T) Y_s, 
    \end{equation}
    and ${f}^{R}_{(X_s, Y_s)}$ is derived from ${f}_{(X_s, Y_s)}$, where  $Z_s$ is constrained additionally by L-Reg and the constrained $Z_s$ is denoted as $Z_s^{R}$:
      \begin{equation}
        {f}^{R}_{(X_s, Y_s)}  = (Z_s^{R\;T} Z_s)^{-1} h(Z_s^{R\;T}) Y_s.
    \end{equation}
    For simplification, we denote  $ (Z^T Z)^{-1} Z^T$, $(Z_s^T Z_s)^{-1} Z_s^T$, and $(Z_s^{R\;T} Z_s)^{-1} Z_s^{R\;T}$ as $\mathcal{N}^*$, $\mathcal{N}$, and $\mathcal{N}^R$, respectively. 
    
    \textbf{The form of  $\mathcal{N}$.} 
    For multi-domain generalization, the model is tested on the unseen domain, referring that $X_u$ contains some unseen semantics besides the seen: $Z_s \sim \mathcal{Z}_s, Z_u \sim \mathcal{Z}_u, \mathcal{Z}_s \neq \mathcal{Z}_u, \mathcal{Z}_s \cap \mathcal{Z}_u \neq \emptyset$. Considering each dimension of $Z$ represents a specific semantics, we denote $\Gamma$ as the dimensions of $Z$ that contain the seen semantics support in $X_s$ and $\Bar{\Gamma}$ for the unseen, we can decompose $\mathcal{N}$ as:
    \begin{equation}
        \mathcal{N} = \begin{bmatrix}
        \Gamma^T \Bar{\Gamma}  & \Gamma^T \Bar{\Gamma} \\
         \Bar{\Gamma}^T \Gamma & \Bar{\Gamma}^T \Bar{\Gamma}
        \end{bmatrix}^{-1} [h(\Gamma)\; h(\Bar{\Gamma})]^T.
    \end{equation}

    \textbf{The form of  $\mathcal{N}^*$.} 
    Assume $\Gamma$ already contains semantic support for deducting $Y$; thus, $\Bar{\Gamma}$ would not affect the deduction of $Y$. In such case, it has that $\Gamma^T \Bar{\Gamma} = \textbf{0}$ and $\Bar{\Gamma}^T\Gamma = \textbf{0}$ and $\Bar{\Gamma}^T \Bar{\Gamma}= \textbf{1}$ where $\textbf{0}, \textbf{1}$ denote zero matrix and identity matrix:
    \begin{equation}
        \begin{split}
            \mathcal{N}^* &= \begin{bmatrix}
        \Gamma^T \Gamma  & \textbf{0} \\
        \textbf{0} & \textbf{1}
        \end{bmatrix}^{-1}  [h(\Gamma)\; h(\Bar{\Gamma})]^T \\ 
        = & \begin{bmatrix}
        (\Gamma^T \Gamma)  & \textbf{0} \\
        \textbf{0} & \textbf{1}
        \end{bmatrix}^{-1} [h(\Gamma)\; h(\Bar{\Gamma})]^T  \\ 
        = &\begin{bmatrix}
        (\Gamma^T \Gamma)^{-1} h(\Gamma)\\
         h(\Bar{\Gamma})
        \end{bmatrix},
        \end{split}
    \end{equation}
    where we also expect $ h(\Bar{\Gamma}) = \textbf{0}$ so that $z_u$ does not influence the deduction.
    We now have $\mathcal{N}^*$:
    \begin{equation}
        \mathcal{N}^* = \begin{bmatrix}
        \Gamma^T \Gamma  & \textbf{0} \\
        \textbf{0} & \textbf{1}
        \end{bmatrix}^{-1} [h(\Gamma)\; h(\Bar{\Gamma})]^T, \; \text{s.t.}, h(\Bar{\Gamma}) = \textbf{0}.
    \end{equation}
    
    Note that for $\mathcal{N}$ in ${f}_{(X_s, Y_s)}$, $\Gamma^T \Bar{\Gamma}$ and $\Bar{\Gamma}^T \Gamma $ are not constrained. Please refer to \cref{lemma:filter}. Furthermore, $h(\Bar{\Gamma})$ is also not constrained.  
    
    \textbf{The form of $\mathcal{N}^{R}$.} Now we discuss the trainable  $\mathcal{N}^{R}$ obtained with the application of L-Reg. 
    The form of  $\mathcal{N}^{R}$ is similar to $\mathcal{N}^*$. However, \cref{eq:reg_goal} indicates that L-Reg minimizes $||\Gamma^T \Bar{\Gamma}||_2$ and $||\Bar{\Gamma}^T \Gamma ||_2$ through $- H(Y|g(\Bar{\Gamma})), D)$ and also minimizing $||h(\Bar{\Gamma})||_2$: 
   \begin{equation}
        \mathcal{N}^{R} = \begin{bmatrix}
      \Gamma^T \Bar{\Gamma}  & \Gamma^T \Bar{\Gamma} \\
         \Bar{\Gamma}^T \Gamma & \Bar{\Gamma}^T \Bar{\Gamma}
        \end{bmatrix}^{-1} [h(\Gamma)\; h(\Bar{\Gamma})]^T,\; \text{s.t.}, \min ||\Gamma^T \Bar{\Gamma}||_2 + ||\Bar{\Gamma}^T \Gamma ||_2 + |h(\Bar{\Gamma})||_2. 
    \end{equation}
    
    \textbf{Compare $ GL({f}^{R}_{(X_s, Y_s)}, f^*, X_u) $ with $   GL({f}_{(X_s, Y_s)}, f^*, X_u)$.}
    By comparing the forms of $ \mathcal{N}^{R},  \mathcal{N}$ and $ \mathcal{N}^*$, it is obvious that $||\mathcal{N}^{R} - \mathcal{N}^*||_2 \leq ||\mathcal{N} - \mathcal{N}^*||_2$. Therefore, we have that:
     $ GL({f}^{R}_{(X_s, Y_s)}, f^*, X_u) \leq GL({f}_{(X_s, Y_s)}, f^*, X_u)$.
\end{proof}

\begin{lemma}[Minimizing $ H(Y|g(X),D) + \mathcal{R}$ solely may cause generalization degradation]
\label{lemma:filter}
Minimizing $ H(Y|g(X),D) + \mathcal{R}$ solely without L-Reg may conflict with $\max_{h,g} H(Y|g(\Bar\Gamma), D)$, causing invalid semantics for decision process and degrading the generalization. 
\end{lemma}

\begin{proof}
We have the following relationship for $H(Y|g(z), D)$: 
\begin{equation}
\begin{gathered}
    H(Y|g(z), D) = H(Y|g(\Bar{\Gamma}),g(\Gamma),D) \\
    H( Y, g(\Gamma)|g(\Bar{\Gamma}),D) - H(g(\Gamma)|g(\Bar{\Gamma}),D)
     = 
     H(Y|g(\Bar{\Gamma}),g(\Gamma),D) + H(g(\Bar{\Gamma})|g(\Gamma),D).
\end{gathered}
\end{equation}
Since the independence between $\{z_i\}_{i=1}^M$ is unconstrained, 
 $ H( Y, g(\Gamma)|g(\Bar{\Gamma}),D)$ may cause that $Y$ can be deducted from $\Bar{\Gamma}$. Therefore, $\Gamma^T \Bar{\Gamma}$ and $ \Bar{\Gamma}^T \Gamma $ are not constrained even when the domain gap is minimized where $|D| = 1$, 
causing the sub-optimal generalization. 
\end{proof}

\section{One toy example}

We present a simplified informal illustrative example to compare the efficacy of our proposed L-Reg against conventional L1 and L2 regularization methods. As depicted in \cref{fig:example}, the ground truth (GT) image represents the underlying data, generated according to $f^*(x_1, x_2) = \sin(2\pi x_1) \cdot \sin(2\pi x_2)$, where $x_1$ and $x_2$ denote the horizontal and vertical coordinates respectively, and the pixel color corresponds to the value of $f^*(x_1, x_2)$. The training domain is delineated by the black box, while the testing domain encompasses the area outside of this boundary.

For our experiments, we use a 6-linear-layer size-110 ReLU model network. Mean squared error serves as the loss function.


Our experimental results reveal that L-Reg enhances the model's ability to extrapolate beyond the training domain. Notably, our proposed L-Reg demonstrates superior extrapolative capabilities compared to traditional $L_1$ and $L_2$ regularization methods. This observation highlights the efficacy of L-Reg in fostering improved generalization.

\begin{figure}[t]
  \centering
  \includegraphics[width=\linewidth]{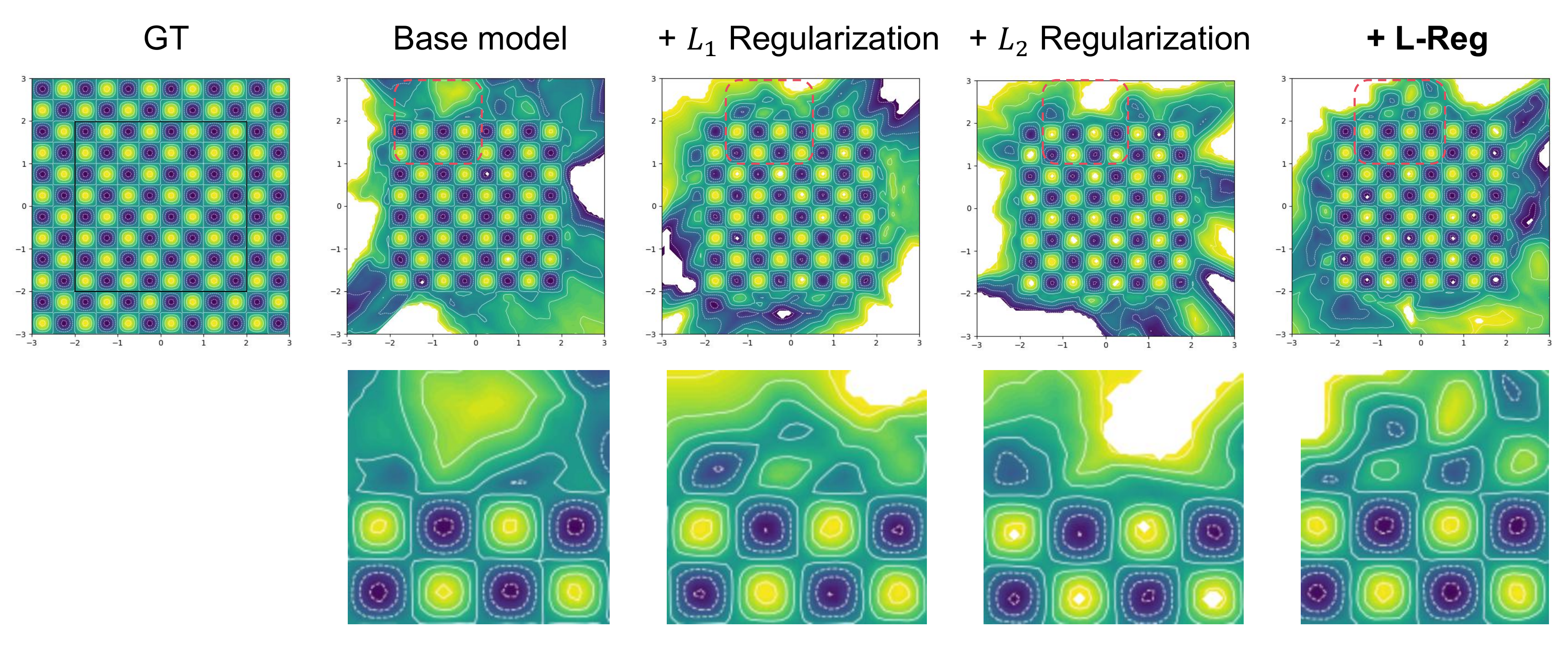}
  \caption{Prediction visualizations of MLP with different regularization terms. }
  \label{fig:example}
\end{figure}

\section{Apply L-Reg to ERM Baseline for mDG} 
\label{app:compare_to_erm}
To further validate L-Reg's efficacy for mDG, we use ERM as the baseline on the TerraIncognita dataset.
For a fair comparison, 
all experiments share the same hyperparameter settings and use the Regnety-16gf backbone. Original ERM results are also included alongside our reproduced results. The results in \cref{tab:more_mDG_ERM_results} reveal that  ERM  with L-Reg significantly improves mDG performance (from 49.9\% to 52.9\%).

\section{Compare L-Reg with more regularization terms} 
\label{app:compare_to_more}
We also compare L-Reg with other regularization terms:  The Ortho-Reg - the orthogonality regularization that constrains the independence of each dimension of the semantic feature $z$; and Sparsity - implemented as Bernoulli Sample of the latent features from the sparse linear concept discovery models \cite{panousis2023sparse} on our used PIM backbone.  
To investigate this fairly, we re-implemented  the Bernoulli
Sample of the latent features from the Sparse Linear Concept Discovery Models \cite{panousis2023sparse} on the same PIM backbone that we used, to achieve the sparsity. 
\cref{tab:more_gcd_compare}\&\cref{tab:more_gcd_compare_full} demonstrate that L-Reg outperforms  Ortho-Reg and Sparsity.

Especially, 
while a common sparse concept model may be able to achieve $\gamma^y\psi = z^y$ by filtering irrelevant features through the sparsity, 
it may not ensure $\gamma^{y_i} \neq \gamma^{y_j}$, which is crucial for disentangling features used for predicting different classes. 
This limitation can potentially lead to degradation in generalization performance for common sparse concept models.
\ref{tab:more_gcd_compare}\&\ref{tab:more_gcd_compare_full} indicate that while L-Reg consistently achieves overall improvement,
the sparse concept-based approach does not consistently improve generalization, validating the aforementioned difference.

\section{Limitation of L-Reg and possible solutions} 
\label{app:limiation}
As analyzed and discussed in the paper, L-Reg is based on the precondition that each dimension of the latent features represents an independent semantic. 

We hypothesize this is due to the fact that our L-Reg is derived based on the precondition that $z^i, z^j \in z, I \neq j$ is independent of each other. This condition holds for most deep-layer features but may not apply to shallow layers. Thus, applying L-Reg to the semantic features from the deep layers may improve the performance for unknown classes without negatively impacting known classes. 

Derived from this hypothesis,
another possible solution is further regularizing the independence, which may lead to further improvements.
To validate this hypothesis, we test L-Reg by reinforcing independence with Ortho-Reg. 
MDG results in \cref{tab:more_mDG_ERM_results} and GCD results in \cref{tab:more_gcd_compare}\&\cref{tab:more_gcd_compare_full} 
show that combining L-Reg with Ortho-Reg leads to further improvements, whereas Ortho-Reg alone may not guarantee improvements. 
These findings support our hypothesis and suggest that L-Reg, particularly when applied to deep layers or in conjunction with Ortho-Reg, is beneficial.
This suggests a direction for future work.


\begin{table}[t]
    \centering
    \captionof{table}{\textbf{Results of GCD:} Averaged results across all datasets of PIM with different regularization applied to the latent features: Sparsity: achieved through Bernoulli Sample; Ortho-Reg: orthogonality regularization. +L-Reg outperforms other regularization terms when they are applied solely; +L-Reg+Ortho-Reg achieves the best performance and alleviates the performance degradation of unknown classes, validating our hypothesis in the paper that the improper $Z$ may result in compromises and constraining the independence of each $z^i\in z, z \in Z$ may be helpful.}
    \label{tab:more_gcd_compare}
    \small{
    \resizebox{0.5\linewidth}{!}{%
    \begin{tabular}{l|ccc}
    \hline
     & \multicolumn{3}{c}{Avg} \\ 
     & All & Known & Unknown \\ \hline
    PIM & 67.4 & 79.3 & 59.9 \\ \hline
    +Sparsity & 66.6 & 77.3 & 60.0 \\
    Improvements & -0.7 & -2.0 & 0.1 \\ \hline
    +Ortho-Reg & 68.4 & 79.2 & 61.9 \\
    Improvements & 1.0 & -0.1 & 2.0 \\ \hline
    \rowcolor{mygray}\textbf{+L-Reg} & 68.8 & 79.0 & 62.7 \\
    \rowcolor{mygray}Improvements & 1.4 & -0.3 & 2.8 \\ \hline
    \rowcolor{mygray}\textbf{+L-Reg+Ortho-Reg} & \textbf{69.3} & \textbf{79.6} & \textbf{63.4} \\
    \rowcolor{mygray}Improvements & \textbf{2.0} & \textbf{0.3} & \textbf{3.5} \\ \hline
    \end{tabular}%
    }
    }
\end{table}

\begin{table}[t]
\caption{\textbf{Results of GCD:} Detailed results across all datasets of PIM with different regularization applied to the latent features: Sparsity: achieved through Bernoulli Sample; Ortho-Reg: orthogonality regularization.}
\label{tab:more_gcd_compare_full}
\centering
\resizebox{\textwidth}{!}{%
\begin{tabular}{l|ccc|ccc|ccc}
\hline
 & \multicolumn{3}{c|}{\textbf{CUB}} & \multicolumn{3}{c|}{\textbf{Stanford Cars}} & \multicolumn{3}{c}{\textbf{Herbarium19}} \\
\multicolumn{1}{c|}{} & All & Known & Unknown & All & Known & Unknown & All & Known & Unknown \\ \hline
PIM & 62.7 & 75.7 & 56.2 & 43.1 & 66.9 & 31.6 & 42.3 & 56.1 & 34.8 \\ \hline
PIM + Sparsity & 60.1 & 72.7 & 53.8 & 40.4 & 61.7 & 30.1 & 42.0 & 53.7 & 35.8 \\
Improvements & -2.6 & -3.0 & -2.4 & -2.7 & -5.2 & -1.5 & -0.3 & -2.4 & 1.0 \\ \hline
PIM + Ortho-Reg & 64.9 & 76.7 & 58.9 & 44.3 & 65.6 & 34.1 & 42.9 & 57.2 & 35.1 \\
Improvements & 2.2 & 1.0 & 2.7 & 1.2 & -1.3 & 2.5 & 0.6 & 1.1 & 0.3 \\ \hline
\rowcolor{mygray}\textbf{PIM + L-Reg} & 65.3 & 76.0 & 60.0 & 44.8 & 66.0 & 34.6 & \textbf{43.7} & 55.8 & \textbf{37.2} \\
\rowcolor{mygray}Improvements & 2.6 & 0.3 & 3.8 & 1.7 & -0.9 & 3.0 & 1.4 & -0.3 & 2.4 \\ \hline
\rowcolor{mygray}\textbf{PIM + L-Reg + Ortho-Reg} & \textbf{66.8} & \textbf{77.3} & \textbf{61.6} & \textbf{45.8} & \textbf{67.3} & \textbf{35.5} & 43.3 & 57.5 & 35.6 \\
\rowcolor{mygray}Improvements & 4.1 & 1.6 & 5.4 & 2.7 & 0.4 & 3.9 & 1.0 & 1.4 & 0.8 \\ \hline \hline
 & \multicolumn{3}{c|}{\textbf{CIFAR10}} & \multicolumn{3}{c|}{\textbf{CIFAR100}} & \multicolumn{3}{c}{\textbf{ImageNet-100}} \\ 
\multicolumn{1}{c|}{} & All & Known & Unknown & All & Known & Unknown & All & Known & Unknown \\ \hline
PIM & 94.7 & 97.4 & 93.3 & 78.3 & 84.2 & 66.5 & 83.1 & \textbf{95.3} & 77.0 \\ \hline 
PIM + Sparsity & 94.2 & 97.4 & 92.6 & 79.7 & \textbf{84.6} & 69.7 & 83.4 & 93.7 & 78.2 \\
Improvements & -0.5 & 0.0 & -0.7 & 1.4 & 0.4 & 3.2 & 0.3 & -1.6 & 1.2 \\ \hline
PIM + Ortho-Reg & 95.1 & 97.4 & 93.9 & 80.2 & \textbf{84.6} & 71.4 & 83.0 & 93.4 & 77.7 \\
Improvements & 0.4 & 0.0 & 0.6 & 1.9 & 0.4 & 4.9 & -0.1 & -1.9 & 0.7 \\ \hline
\rowcolor{mygray}\textbf{PIM + L-Reg} & 94.8 & \textbf{97.6} & 93.4 & 80.8 & \textbf{84.6} & 73.2 & 83.4 & 94.0 & 78.0 \\
\rowcolor{mygray}Improvements & 0.1 & 0.2 & 0.1 & 2.5 & 0.4 & 6.7 & 0.3 & -1.3 & 1.0 \\ \hline
\rowcolor{mygray}\textbf{PIM + L-Reg + Ortho-Reg} & \textbf{95.1} & \textbf{97.6} & \textbf{93.9} & \textbf{81.2} & 84.2 & \textbf{75.0} & \textbf{83.7} & 93.6 & \textbf{78.7} \\
\rowcolor{mygray}Improvements & 0.4 & 0.2 & 0.6 & 2.9 & 0.0 & 8.5 & 0.6 & -1.7 & 1.7 \\
\hline
\end{tabular}%
}
\end{table}

\begin{table}[t]
\caption{\textbf{Results of mDG:} Results of using ERM as the baseline.  
We use the ERM method as the baseline to test L-Reg's efficacy. 
Ortho-Reg: orthogonality regularization. 
This table includes results: (1) The improved performance of L-Reg on ERM baseline. (2) Comparison between L-Reg and the Ortho-Reg on ERM baseline. 
(3) Using L-Reg and Ortho-Reg together yields further promotion, validating our `improper $z$' hypothesis in the Paper limitation part. 
The used dataset is TerraIncognita. All experiments share the same hyperparameters except the added regularization term. Each group of experiments is run with seeds [0,1,2], and the averaged results for each domain and additionally with the standard deviation (Std) are reported.}
\label{tab:more_mDG_ERM_results}
\resizebox{\textwidth}{!}{%
\centering
\begin{tabular}{l|ccccc}
\hline
\textbf{TerraIncognita} & \textbf{Location 100} & \textbf{Location 38} & \textbf{Location 43} & \textbf{Location 46} & \textbf{Avg ± Std.} \\ \hline
ERM & 54.3 & 42.5 & 55.6 & 38.8 & 47.8 \\
ERM Reproduced & 50.6 & 49.7 & 58 & 41.2 & 49.9±3.6 \\
+Ortho-Reg & 50.7 & 52.6 & 60.5 & 42.7 & 51.6±2.5 \\ \hline
\rowcolor{mygray}\textbf{+L-Reg} & 52.7 & 51.7 & 61.3 & 45.8 & 52.9±4.2 \\
\rowcolor{mygray}\textbf{+L-Reg+Ortho-Reg} & 61.5 & 48.6 & 60.3 & 44 & 53.6±0.5 \\ \hline
\end{tabular}%
}
\end{table}

\section{More experimental details and results}
\label{app:experiments}

All experiments can be conducted on one NVIDIA GeForce RTX 3090 GPU.

\subsection{Multi-domain generalization}
\label{app:mDG}

\textbf{Competitors.}
We listed results from previous important work in the mDG field for better validation. They are:  
MMD~\citep{li2018domainMMD},
Mixstyle~\citep{zhou2021domain},
GroupDRO~\citep{sagawa2019distributionally},
IRM~\cite{arjovsky2019invariant},
ARM~\citep{zhang2021adaptive},
VREx~\citep{krueger2021out},
CDANN~\citep{li2018deep},
DANN~\citep{ganin2016domain},
RSC~\citep{huang2020self},
MTL~\citep{blanchard2021domain},
MLDG~\citep{li2018learning},
Fish~\citep{shi2021gradient},
ERM~\citep{Vapnik1998ERM},
SagNet~\citep{nam2021reducing},
SelfReg~\citep{kim2021selfreg},
CORAL~\citep{sun2016deep},
mDSDI~\cite{bui2021exploiting},
MIRO~\cite{cha2022miro},
and
{GMDG}~\cite{tan2024rethinking}. Among them, GMDG is treated as our baseline since it sufficiently minimizes the domain gaps.

\textbf{Datasets.}
We use PACS (4 domains, 9,991
samples, $7$ classes)~\citep{li2017deeper}, VLCS ($4$ domains, $10,729$ samples, $5$ classes)~\citep{fang2013unbiased}, OfficeHome (4 domains, 15,588 samples, 65 classes)~\citep{venkateswara2017deep}, TerraIncognita (TerraIncognita, $4$ domains, $24,778$ samples, $10$ classes)~\citep{beery2018recognition}, and DomainNet (6 domains,
586,575 samples, 345 classes)~\citep{peng2019moment}.

\textbf{Training details.} 
We use GMDG~\cite{tan2024rethinking} as our baseline. Especially, we use all loss terms proposed in GMDG as $L_{main}$.
The training procedure is the same as MIRO~\cite{cha2022miro} and GMDG.
We use seeds $0,1,2$ for all three trails training. 
\begin{table}[t]
    \centering
    \caption{Parameters for mDG task}
    \label{tab:mDG_paras}
    \begin{tabular}{lccccccccc}
    \toprule
         Use RegNetY-16GF & lr mult & $\alpha$ \\ \hline
         TerraIncognita &   2.5 &  1e-3 \\
         OfficeHome &  0.1 & 1e-3  \\
         VLCS &  0.1 & 1e-4   \\
         PACS &  0.1 & 5e-4   \\ 
         DomainNet & 5.0 & 1e-3 \\
    \bottomrule     
    \end{tabular}
\end{table}

\textbf{Parameters.} 
We adhere to the parameters proposed by GMDG, particularly focusing on its recommended loss terms. Furthermore, we provide a detailed listing of the hyper-parameters pertaining to L-Reg, along with the tuned `lr mult', as outlined in \cref{tab:mDG_paras}, to facilitate the reproducibility of our results.

\textbf{Evaluation metric.}
The models undergo training on known domains and subsequent testing on unseen domains. For each trial, a distinct domain within the datasets is designated as the unseen domain. The evaluation metric reports the prediction accuracy achieved on these unseen domains. The aggregated results across all unseen domains within the datasets provide a comprehensive assessment of the algorithm's performance in domain generalization for the given datasets.

\textbf{More results.}
Results of each domain for each dataset are presented in \cref{tab:mDG_res_detial1,tab:mdg_res_detial2,tab:mdg_res_detial3,tab:mdg_res_detial4,tab:mdg_res_detial5}.

\begin{table}[t]
\centering
\caption{MDG experiments on TerraIncognita: More results of full GMDG+L-Reg for each category.}
\label{tab:mDG_res_detial1}
\resizebox{\linewidth}{!}{%
\begin{tabular}{l|cccc|c}
\toprule
 \textbf{TerraIncognita} & \multicolumn{1}{c}{\textbf{Location 100}} & \multicolumn{1}{c}{\textbf{Location 38}} & \multicolumn{1}{c}{\textbf{Location 43}} & \multicolumn{1}{c|}{\textbf{Location 46}} & \multicolumn{1}{c}{\textbf{Avg.}} \\ \midrule
ERM~\cite{gulrajani2020search} & 54.3 & 42.5 & 55.6 & {38.8} & 47.8\\
MIRO~\cite{cha2022miro} (use ResNet-50) & - & -&- &- &{50.4} \\
{GMDG}~\cite{tan2024rethinking} (use ResNet-50) & {59.8±1.0} & {45.3±1.7} & {57.1±1.8} & 38.2±5 & 50.1±1.2 \\ 
\midrule
MIRO~\cite{cha2022miro}  (use RegNetY-16GF) & - & -&- &- &58.9±1.3 \\
{GMDG}~\cite{tan2024rethinking} (use RegNetY-16GF) & 73.3±3.3 & 54.7±1.4 & 67.1±0.3 & 48.6±6.5 & 60.7±1.8 \\ \midrule
\rowcolor{mygray}{GMDG} + \textbf{L-Reg} (use RegNetY-16GF)& 
\textbf{73.9}±0.8&	\textbf{57.1}±2.3&	\textbf{67.9}±1.1&	\textbf{52.7}±4.0&	\textbf{62.9}±0.9
\\ 
\bottomrule
\end{tabular}%
}
\vspace{0.5cm}
\caption{MDG experiments on OfficeHome: More results of full GMDG+L-Reg for each category.}
\label{tab:mdg_res_detial2}
\centering
\resizebox{\linewidth}{!}{%
\begin{tabular}{l|cccc|c}
\toprule
 \textbf{OfficeHome} & \multicolumn{1}{c}{\textbf{art}} & \multicolumn{1}{c}{\textbf{clipart}} & \multicolumn{1}{c}{\textbf{product}} & \multicolumn{1}{c}{\textbf{real}} & \multicolumn{1}{c}{\textbf{Avg.}}  \\ \midrule
ERM~\cite{gulrajani2020search} & 63.1 & 51.9  & 77.2&  78.1& 67.6\\
MIRO~\cite{cha2022miro} (use ResNet-50) & - & -&- &- &70.5±0.4 \\
{GMDG}~\cite{tan2024rethinking} (use ResNet-50) & 68.9±0.3 & 56.2±1.7 & 79.9±0.6 & 82.0±0.4 & 70.7±0.2  \\
\midrule
MIRO~\cite{cha2022miro}  (use RegNetY-16GF) & - & -&- &- & 80.4±0.2 \\
{GMDG}~\cite{tan2024rethinking} (use RegNetY-16GF) & \textbf{79.7}±1.6 & 67.7±1.8 & 87.8±0.8 & 87.9±0.7 & 80.8±0.6 \\ \midrule
\rowcolor{mygray}{GMDG} + \textbf{L-Reg} (use RegNetY-16GF)& 
78.4±0.3	&\textbf{69.3}±0.7&	\textbf{87.9}±0.6	&\textbf{88.0}±0.8	&\textbf{80.9}±0.5
\\
\bottomrule
\end{tabular}%
}
\vspace{0.5cm}
\caption{MDG experiments on VLCS: More results of full GMDG+L-Reg for each category.}
\label{tab:mdg_res_detial3}
\centering
\resizebox{\linewidth}{!}{%
\begin{tabular}{l|cccc|c}
\toprule
\textbf{VLCS} & \multicolumn{1}{c}{\textbf{caltech101}} & \multicolumn{1}{c}{\textbf{labelme}} & \multicolumn{1}{c}{\textbf{sun09}} & \multicolumn{1}{c|}{\textbf{voc2007}} & \multicolumn{1}{c}{\textbf{Avg.}}\\ \midrule
ERM~\cite{gulrajani2020search}& 97.7 &64.3& 73.4 &74.6 &77.3\\
MIRO~\cite{cha2022miro} (use ResNet-50) & - & -&- &- & 79.0±0.0 \\
{GMDG}~\cite{tan2024rethinking} (use ResNet-50)& 98.3±0.4 & 65.9±1 & 73.4±0.8 & 79.3±1.3 & 79.2±0.3\\
\midrule
MIRO~\cite{cha2022miro}  (use RegNetY-16GF) & - & -&- &- & 79.9±0.6\\
{GMDG}~\cite{tan2024rethinking} (use RegNetY-16GF) & 97.9±1.3 & {66.8±2.1} & {\textbf{80.8}±1} & \textbf{83.9}±1.8 & {82.4±0.6} 
\\ \midrule
\rowcolor{mygray}{GMDG} + \textbf{L-Reg} (use RegNetY-16GF)  & 
\textbf{98.6}±0.1	&\textbf{67.1}±0.1	&80.7±0.7	&83.0±0.8	&\textbf{82.4}±0.0
\\ 
\bottomrule
\end{tabular}%
}
\vspace{0.5cm}
\caption{MDG experiments on PACS: More results of full GMDG+L-Reg for each category.}
\label{tab:mdg_res_detial4}
\centering
\resizebox{\linewidth}{!}{%
\begin{tabular}{l|cccc|c}
\toprule
 \textbf{PACS} & \multicolumn{1}{c}{\textbf{art\_painting}} & \multicolumn{1}{c}{\textbf{cartoon}} & \multicolumn{1}{c}{\textbf{photo}} & \multicolumn{1}{c|}{\textbf{sketch}} & \multicolumn{1}{c}{\textbf{Avg.}}  \\ \midrule
ERM~\cite{gulrajani2020search}  &84.7 & 80.8 & 97.2 & 79.3  &84.2\\
MIRO~\cite{cha2022miro} (use ResNet-50)& - & -&- &- & 85.4±0.4 \\
{GMDG}~\cite{tan2024rethinking} (use ResNet-50) & 84.7±1.0 & 81.7±2.4 & 97.5±0.4 & 80.5±1.8 & 85.6±0.3\\
\midrule
MIRO~\cite{cha2022miro}  (use RegNetY-16GF)& - & -&- &- & 97.4±0.2\\
{GMDG}~\cite{tan2024rethinking} (use RegNetY-16GF) & 97.5±1.0 & 97.0±0.2 & \textbf{99.4}±0.2 & 95.2±0.4 & 97.3±0.1
\\ \midrule
\rowcolor{mygray}{GMDG} + \textbf{L-Reg} (use RegNetY-16GF) & 
\textbf{97.6}±0.8& 	\textbf{97.1}±0.3& 	99.3±0.2& 	\textbf{95.3}±0.9& 	\textbf{97.4}±0.2
\\ 
\bottomrule
\end{tabular}%
}
\vspace{0.5cm}
\caption{MDG experiments on DomainNet: More results of full GMDG+L-Reg for each category.}
\label{tab:mdg_res_detial5}
\centering
\resizebox{\linewidth}{!}{%
\begin{tabular}{l|cccccc|c}
\toprule
 \textbf{DomainNet} &  \multicolumn{1}{c}{\textbf{clipart}} & \multicolumn{1}{c}{\textbf{info}} & \multicolumn{1}{c}{\textbf{painting}} & \multicolumn{1}{c}{\textbf{quickdraw}} & \multicolumn{1}{c}{\textbf{real}} & \multicolumn{1}{c|}{\textbf{sketch}} & \multicolumn{1}{c}{\textbf{Avg.}} \\ \midrule
ERM~\cite{gulrajani2020search}&50.1 &63.0& 21.2 &63.7& 13.9& 52.9 &44.0 \\
MIRO~\cite{cha2022miro} (use ResNet-50)& - & -&- &- & - & -& 44.3±0.2\\
\textbf{GMDG} (use ResNet-50) & 63.4±0.3 & 22.4±0.4 & 51.4±0.4 & 13.4±0.8 & 64.4±0.3 & 52.4±0.4 & 44.6±0.1 \\
\midrule
MIRO~\cite{cha2022miro}  (use RegNetY-16GF) & - & -&- &- & - & -& 53.8±0.1\\
{GMDG} (use RegNetY-16GF) & 74.0±0.3 & 39.5±1.5 & 61.5±0.3 & \textbf{16.3}±1.2 & 73.9±1.5 & 62.8±2.4 & 54.6±0.1 
\\ \midrule
\rowcolor{mygray}{GMDG} + \textbf{L-Reg} (use RegNetY-16GF)  & 
\textbf{74.1}±0.1	&\textbf{42.6}±1.0	&\textbf{62.3}±2.9	&12.7±0.9	&\textbf{75.9}±0.8	&\textbf{64.6}±0.2	&\textbf{55.4}±0.0
\\ 
\bottomrule
\end{tabular}%
}
\end{table}

\begin{table}[t]
    \centering
    \caption{Statistics of datasets.}
    \label{tab:stats}
    \resizebox{\linewidth}{!}{
    \begin{tabular}{l|cccccc}
        \toprule
                & CUB  & Standford Cars & Herbarium19 & CIFAR10 & CIFAR100 & ImageNet-100 \\
        \hline
        Known classes & 100  & 98             & 341         & 5       & 80       & 50           \\
        Seen data & 1.5K & 2.0K           & 8.9K        & 12.5K   & 20K      & 31.9K        \\ \hline
        All classes & 200  & 196            & 683         & 10      & 100      & 100          \\
        Unseen data & 4.5K & 6.1K           & 25.4K       & 37.5K   & 30K      & 95.3K        \\
        \bottomrule
    \end{tabular}
    }
\vspace{0.5cm}
\caption{Tuned weight decay values for each dataset.}
\label{tab:tuned_wd_a}
\resizebox{\columnwidth}{!}{%
\begin{tabular}{l|cccccc}
\toprule 
                     & CUB    & Standford Cars & Herbarium19 & CIFAR10 & CIFAR100 & ImageNet-100 \\ \hline
Tuned weighted decay & 0.02/2 & 0.02/2         & 0.02/2      & 0.05/2  & 0.005/2  & 0.005/2     \\ 
$\alpha$ of L-Reg&  0.1 & 0.001 & 0.2 & 0.01 & 0.0025 & 0.01
\\ 
\bottomrule
\end{tabular}%
}
\vspace{0.5cm}
\centering
\caption{GCD results: Accuracy scores 
across fine-grained and generic PIM datasets with our L-Reg and other competitors. The best results of each group are highlighted in \textbf{bold}. 
Improvement and degradation in our approach from PIM are highlighted in \red{red} and \blue{blue}, respectively.
}
\label{tab:GCD_res_details}
\resizebox{\linewidth}{!}{%
    \begin{tabular}{l|lll|lll|lll}
        \toprule
        & \multicolumn{3}{c|}{\textbf{CUB}}     & \multicolumn{3}{c|}{\textbf{Stanford Cars}} & \multicolumn{3}{c}{\textbf{Herbarium19}}                                                                                                  \\
        \hline Approach & \multicolumn{1}{l}{{All}} & \multicolumn{1}{l}{{Known}} & \multicolumn{1}{l|}{{Unknown}} & \multicolumn{1}{l}{{All}} & \multicolumn{1}{l}{{Known}} & \multicolumn{1}{l|}{{Unknown}} & \multicolumn{1}{l}{{All}} & \multicolumn{1}{l}{{Known}} & \multicolumn{1}{l}{{Unknown}}           \\
        \hline
        K-means                                                         & 34.3                        & 38.9                              & 32.1                             & 12.8          & 10.6          & 13.8          & 12.9          & 12.9          & 12.8          \\
        RankStats+ \cite{han2021autonovel} (TPAMI-21)                   & 33.3                        & 51.6                              & 24.2                             & 28.3          & 61.8          & 12.1          & 27.9          & 55.8          & 12.8          \\
        UNO+ \cite{fini2021unified} (ICCV-21)                           & 35.1                        & 49.0                              & 28.1                             & 35.5          & \textbf{70.5} & 18.6          & 28.3          & \textbf{53.7}          & 14.7          \\
        ORCA \cite{cao2022openworld} (ICLR-22)                          & 27.5                        & 20.1                              & 31.1                             & 15.9          & 17.1          & 15.3          & 22.9          & 25.9          & 21.3          \\ 
        ORCA \cite{cao2022openworld} - ViTB16                           & 38.0                        & 45.6                              & 31.8                             & 33.8          & 52.5          & 25.1          & 25.0          & 30.6          & 19.8          \\ %
        GCD \cite{vaze2022generalized} (CVPR-22)                        & \textbf{51.3}                        & \textbf{56.6}                              & \textbf{48.7}                             & \textbf{39.0}          & 57.6          & \textbf{29.9}          & \textbf{35.4}          & 51.0          & \textbf{27.0}          \\
        \hline
        \multicolumn{1}{c}{}& \multicolumn{9}{c}{InfoMax based methods} \\ \hline
        RIM \cite{krause2010discriminative} (NeurIPS-10)  & 52.3                        & 51.8                              & 52.5                             & 38.9          & 57.3          & 30.1          & 40.1          & \textbf{57.6} & 30.7          \\
        TIM \cite{boudiaf2020information} (NeurIPS-20)                  & 53.4                        & 51.8                              & 54.2                             & 39.3          & 56.8          & 30.8          & 40.1          & 57.4          & 30.7          \\ \hline
        PIM~\cite{chiaroni2023parametric} (ICCV-23)                                                           & 62.7                        & \textbf{75.7 }                             & 56.2                             & 43.1          & \textbf{66.9}          & 31.6          & 42.3          & 56.1          & 34.8          \\
        \rowcolor{mygray}\textbf{PIM + L-Reg (Ours)}  
        &\textbf{65.3}\red{$^{2.6\uparrow}$}
        &\textbf{76.0} \red{$^{0.3\uparrow}$}
        &\textbf{60.0}\red{$^{3.8\uparrow}$}
        &\textbf{44.8}\red{$^{1.7\uparrow}$}
        &{66.0}\blue{$^{1.4\downarrow}$}
        &\textbf{34.6}\red{$^{3.0\uparrow}$}	&\textbf{43.7}\red{$^{2.4\uparrow}$}
        &55.8\blue{$^{0.3\downarrow}$}
        &\textbf{37.2}\red{$^{1.6\uparrow}$} \\ 
        \toprule
        \toprule
        & \multicolumn{3}{c|}{\textbf{CIFAR10}} & \multicolumn{3}{c|}{\textbf{CIFAR100}}      & \multicolumn{3}{c}{\textbf{ImageNet-100}}                                                                                                 \\
        \hline Approach & \multicolumn{1}{l}{{All}} & \multicolumn{1}{l}{{Known}} & \multicolumn{1}{l|}{{Unknown}} & \multicolumn{1}{l}{{All}} & \multicolumn{1}{l}{{Known}} & \multicolumn{1}{l|}{{Unknown}} & \multicolumn{1}{l}{{All}} & \multicolumn{1}{l}{{Known}} & \multicolumn{1}{l}{{Unknown}}           \\
        \hline
        K-means                                                         & 83.6                        & 85.7                              & 82.5                             & 52.0          & 52.2          & 50.8          & 72.7          & 75.5          & 71.3          \\
        RankStats+ \cite{han2021autonovel} (TPAMI-21)                   & 46.8                        & 19.2                              & 60.5                             & 58.2          & \textbf{77.6}          & 19.3          & 37.1          & 61.6          & 24.8          \\
        UNO+ \cite{fini2021unified} (ICCV-21)                           & 68.6                        & \textbf{98.3}                     & 53.8                             & 69.5          & 80.6          & 47.2          & 70.3          & \textbf{95.0}          & 57.9          \\
        ORCA \cite{cao2022openworld} (ICLR-22)                          & 88.9                        & 88.2                              & 89.2                             & 55.1          & 65.5          & 34.4          & 67.6          & 90.9          & 56.0          \\ 
        ORCA \cite{cao2022openworld} - ViTB16                           & \textbf{97.1}               & 96.2                              & \textbf{97.6}                    & 69.6          & {76.4}          & 56.1          & \textbf{76.5}          & {92.2}          & \textbf{68.9}          \\ %
        GCD \cite{vaze2022generalized} (CVPR-22)                        & 91.5                        & 97.9                              & 88.2                             & \textbf{70.8}          & \textbf{77.6}          & \textbf{57.0}          & 74.1          & 89.8          & 66.3          \\
        \hline
        \multicolumn{1}{c}{} & \multicolumn{9}{c}{InfoMax based methods} \\ \hline
         RIM \cite{krause2010discriminative} (NeurIPS-10)   & 92.4                        & \textbf{98.1}                              & 89.5                             & 73.8          & 78.9          & 63.4          & 74.4          & 91.2          & 66.0          \\
        TIM \cite{boudiaf2020information} (NeurIPS-20)                  & 93.1                        & 98.0                              & 90.6                             & 73.4          & 78.3          & 63.4          & 76.7          & 93.1          & 68.4          \\
        \hline
        PIM~\cite{chiaroni2023parametric} (ICCV-23)                                                          & 94.7                        & 97.4                              & 93.3                             & 78.3          & 84.2          & 66.5          & 83.1          & \textbf{95.3} & 77.0          \\
        \rowcolor{mygray} \textbf{PIM + L-Reg(Ours)}   
        &\textbf{94.8}\red{$^{0.1\uparrow}$}	& \textbf{97.6} \red{$^{0.2\uparrow}$} 	&\textbf{93.4}\red{$^{0.1\uparrow}$} &\textbf{80.8}\red{$^{2.5\uparrow}$}	&\textbf{84.6}\red{$^{0.2\uparrow}$}	&\textbf{73.2}\red{$^{6.7\uparrow}$} &\textbf{83.4}\red{$^{0.3\uparrow}$} &94.0\blue{$^{1.3\downarrow}$} &\textbf{78.0}\red{$^{1.0\uparrow}$}\\ 
        \bottomrule
    \end{tabular}
}
\vspace{0.5cm}
    \caption{GCD results: Accuracy scores across fine-grained and generic datasets of each setting.
        The best results are highlighted in \textbf{bold}. 
        To eliminate the impact of hyper-parameters on performance, we also present the results of PIM with tuned hyper-parameters (termed baseline tuned).
        ${L}_{main}$ denotes the losses used in PIM. 
        $g$ denotes the transformation applied to the input features.
    }
    \label{tab:gcd_ablation_res}
    \resizebox{\linewidth}{!}{%
        \begin{tabular}{llccc|ccc|ccc}
            \toprule
                                 &                                                                             &  \multicolumn{3}{c|}{\textbf{CUB}}     & \multicolumn{3}{c|}{\textbf{Stanford Cars}} & \multicolumn{3}{c}{\textbf{Herbarium19}}                                                                                                                                                                                                                                                                                                                                                 \\   \hline
            ID                   & Settings                                                                                    & \multicolumn{1}{c}{{\;\;\;\;\;\;All\;\;\;\;\;\;}} & \multicolumn{1}{c}{{Known}}         & \multicolumn{1}{c|}{{Unknown}}      & \multicolumn{1}{c}{{\;\;\;\;\;\;All\;\;\;\;\;\;}} & \multicolumn{1}{c}{{Known}}         & \multicolumn{1}{c|}{{Unknown}}       & \multicolumn{1}{c}{{\;\;\;\;\;\;All\;\;\;\;\;\;}} & \multicolumn{1}{c}{{Known}}   & \multicolumn{1}{c}{{Unknown}}       \\ \hline
            1                    & Baseline (${L}_{main}$)   & 62.7                                              & 75.7                                & 56.2                                & 43.1                                              & 66.9                                & 31.6                                 & 42.3                                              & 56.1                          & 34.8                                \\ \hline
            2                    & Baseline tuned (${L}_{main}$)                               & 64.8                                              & 75.1                                & 59.6                                & 42.6                                              & 59.3                                & 34.6                                 & 43.1                                              & 57.6                          & 35.4                                \\
            6                    & ${L}_{main}$ + $g$                              & 64.9                                              & \textbf{76.7}                                & 58.9                                & 44.7                                              & 65.8                                & 34.6                                 & 43.0                                              & \textbf{57.4}                          & 35.2                                \\
            \rowcolor{mygray} 9  & \textbf{Ours} (${L}_{main}$+$h$+${L}_{L-Reg}$) &  \textbf{65.3}  &    76.0   & \textbf{60.0} & \textbf{44.8}  &     \textbf{66.0}    & \textbf{34.6}  &  \textbf{43.7}    &  55.8  &  \textbf{37.2}

            \\
            \toprule
            \toprule
                                 &                                                                             & \multicolumn{3}{c|}{\textbf{CIFAR10}} & \multicolumn{3}{c|}{\textbf{CIFAR100}}      & \multicolumn{3}{c}{\textbf{ImageNet-100}}                                                                                                                                                                                                                                                                                                                                                \\
            \hline
            ID                   & Settings                                                                                   & \multicolumn{1}{c}{{All}}                         & \multicolumn{1}{c}{{Known}}         & \multicolumn{1}{c|}{{Unknown}}      & \multicolumn{1}{c}{{All}}                         & \multicolumn{1}{c}{{Known}}         & \multicolumn{1}{c|}{{Unknown}}       & \multicolumn{1}{c}{{All}}                         & \multicolumn{1}{c}{{Known}}   & \multicolumn{1}{c}{{Unknown}}       \\
            \hline
            1                    & Baseline (${L}_{main}$)                                                & 94.7                                              & 97.4                                & 93.3                                & 78.3                                              & 84.2                                & 66.5                                 & 83.1                                              & \textbf{95.3}                 & 77.0                                \\ \hline
            2                    & Baseline tuned (${L}_{main}$)                                               & 95.0                                              & 96.1                                & 94.4                                & 80.3                                              & 84.6                                & 71.8                                 & 83.5                                              & 95.0                          & 77.7                                \\
            6                    & ${L}_{main}$+ $g$                                                               & 94.7                                              & 97.5                                & 93.3                                & 80.8                                              & 84.6                                & 73.1                                 & 83.1                                              & 95.0                          & 77.1                                \\
            \rowcolor{mygray} 9  & \textbf{Ours} (${L}_{main}$+$g$+${L}_{L-Reg}$)                      & \textbf{94.8}   &   \textbf{97.6}   & \textbf{93.4}& \textbf{80.8}   &   \textbf{84.6}  &  \textbf{73.2} & \textbf{83.4}  &   94.0   &  \textbf{78.0}

            \\
            \bottomrule
        \end{tabular}%
    }
\end{table}

\subsection{Generalized category discovery}
\label{app:GCD}
\textbf{Competitors.}
We compare our proposed method with existing generalized category discovery methods:
GCD~\cite{vaze2022generalized},
and
PIM~\cite{chiaroni2023parametric}.
In particular, PIM based on information maximization is the current state-of-the-art (SOTA) generalized category discovery method.
Additionally, the traditional machine learning method, k-means~\cite{macqueen1967classification}; three novel category discovery methods:
RankStats+~\cite{han2021autonovel},
UNO+~\cite{fini2021unified},
ORCA~\cite{cao2022openworld}; and several information maximization methods:
RIM~\cite{krause2010discriminative},
and TIM~\cite{boudiaf2020information}
are adapted for generalized category discovery as competitors.
The results of the modified novel category discovery methods are reported in~\cite{vaze2022generalized}, and the modified information maximization methods are reported in~\cite{chiaroni2023parametric}.

\textbf{Usage details of datasets for GCD.}
Following the protocols of GCD and PIM \cite{vaze2022generalized, chiaroni2023parametric}, the initial training set of each dataset is divided into labeled and unlabeled subsets; samples from half of the classes are assigned as unlabeled, and their labels are not used for training. 
Specifically, half of the image samples from known classes are allocated to the labeled subset, while the remaining half are assigned to the unlabeled subset. 
Additionally, the unlabeled subset includes all image samples from the novel classes in the original dataset. As a result, the unlabeled subset consists of instances from $K$ different classes. The detailed statistics of datasets are listed in \cref{tab:stats}.

{\textbf{Training details.} Consistent with PIM, we utilize latent features extracted by the feature encoder DINO (VIT-B/16)~\cite{caron2021emerging} that is pre-trained on ImageNet~\cite{deng2009imagenet} through self-supervised learning. 
The losses proposed in PIM are treated as $L_{main}$. 
The original PIM freezes the feature extractor during the training, directly using the pre-saved extracted features as the model input. For a fair comparison, we only added one linear layer as $g$ on the extracted features, which is the minimal modification.

}

\textbf{$L_2$ (weight decay) value searching.}
For a more fair comparison, we conduct weight decay value searching to ensure that the weight of $L_2$ is the best. 
To address this, we devised a methodology for weight decay searching involving the construction of smaller labeled and unlabeled subsets derived solely from the labeled data.
To conduct parameter searching, we split the labeled samples to construct a 'smaller' sub-labeled and sub-unlabeled set. Specifically, we take $50\%$ of the samples from known classes as sub-unlabeled samples from unknown classes. Additionally, we take $25\%$ of the samples from the remaining $50\%$ of known classes as sub-unlabeled samples from known classes. The remaining samples are treated as sub-labeled samples. Hyper-parameters are then searched on these sub-labeled and sub-unlabeled sets.

\textbf{Parameters of L-Reg.}
The hyper-parameters of L-Reg values are shown in \cref{tab:tuned_wd_a}.

\textbf{Evaluation metric.}
{Following prior works \cite{vaze2022generalized,chiaroni2023parametric}, we use the proposed accuracy metric from \cite{vaze2022generalized} of all classes, known classes, and unknown classes for evaluation. 
}

\textbf{More results.} 
The results for each dataset are presented in \cref{tab:GCD_res_details}. It is evident that L-Reg yields enhanced performance across half of the datasets for both known and unknown classes. On the remaining datasets, while L-Reg may slightly compromise the performance of known classes, it demonstrates significant improvements in the unknown classes, resulting in an overall enhancement in the performance across all classes.


\textbf{More ablation results.} 
Due to the introduction of tuned weight decay and the additional $g$ component, we have conducted ablation studies to assess their impact. The results are summarized in \cref{tab:gcd_ablation_res}. It is observed that the baseline model utilizing the tuned weight decays performs slightly better than the original weight decay settings. Notably, the tuned weight decays contribute to improvements in unknown classes while often leading to slight decreases in known classes across most datasets. 
Inclusion of the proposed extra component $g$ results in marginal improvements in both known and unknown classes compared to the tuned baseline.
Our proposed L-Reg demonstrates significant improvements specifically in the unknown classes, thereby corroborating \cref{prop:target}. However, as discussed in the main paper, it is acknowledged that L-Reg may entail compromises in the performance of known classes.

\subsection{Combination of multi-domain generalization and generalized category discovery}
\label{app:mDG+GCD}
\textbf{Datasets.} 
We leverage the datasets utilized in mDG tasks to construct the mDG+GCD datasets. Specifically, during the seen domains of training, labels from approximately half of the classes are masked. For instance, in the PCAS dataset comprising $7$ classes, classes labeled within the range $[0, 1, 2, 3]$ are retained, while classes in $[4, 5, 6]$ are masked.
It is noteworthy that data categorized as unknown classes in our setup are from unknown classes. 
However, we acknowledge that this prior is not explicitly known. \textbf{To align with the GCD setting, we operate under the assumption that the unlabeled set may potentially include samples from known classes.} Consequently, we refrain from constraining the model by mandating that unlabeled data be classified solely as unknown classes. This adjustment introduces a more challenging generalization scenario.

\textbf{Training details.}
For all experiments, the implementation directly adds L-Reg to their previously proposed loss sets.  The models are trained with the aforementioned labeled and unlabeled sets from the seen domains and tested on the samples from the unseen domain.


\textbf{Parameters.} We include all the parameters for reproducing our experiments in the code. Please refer to the code for details.  

\textbf{Evaluation metric.}
We use the same metric from the GCD task for the mDG+GCD task. Similarly, the metrics include the accuracy for all, known and unknown classes.   

\textbf{More results.}
The averaged results of each dataset are exhibited in \cref{tab:avg_mdg_gcd_more_results}, while the detailed results of each dataset are presented in \cref{tab:mdg_gcd_more_results1,tab:mdg_gcd_more_results2,tab:mdg_gcd_more_results3,tab:mdg_gcd_more_results4,tab:mdg_gcd_more_results5}.

\begin{table}[t]
\caption{MDG+GCD results: accuracy scores of each dataset. Improvements are highlighted in \red{red}.}
\label{tab:avg_mdg_gcd_more_results}
\resizebox{\linewidth}{!}{%
\begin{tabular}{@{}l|l|ccc|ccc|ccc|ccc|ccc@{}}
\toprule
             &                      & \multicolumn{3}{c|}{\textbf{PACS}}                                                         & \multicolumn{3}{c|}{\textbf{HomeOffice}}                                                   & \multicolumn{3}{c|}{\textbf{VLCS}}                                                         & \multicolumn{3}{c|}{\textbf{TerraIncognita}}                                                & \multicolumn{3}{c}{\textbf{DomainNet}}                                                     \\ 
Method       & Domain gap           & {All}  & {Known}  & {Unknown} & {All}  & {Known}  & {Unknown} & {All}  & {Known}  & {Unknown} & {All}  & {Known}  & {Unknown}        & {All}  & {Known}  & {Unknown}         \\ \midrule
ERM           & Not         & 57.26                        & 77.77                        & 22.33                        & 44.80                        & 74.67                        & 8.50                         & 61.51                        & 82.89                        & 34.88                        & 37.34                        & 20.46                        & 45.15                         & 22.56                        & 40.89                        & 6.85                         \\
\textbf{+L-Reg}        &   minimized                   & 55.86                        & 77.69                        & 19.06                        & 43.56                        & 71.78                        & 9.68                         & 61.49                        & 81.33                        & 36.65                        & 40.73                        & 29.27                        & 35.56                         & 25.86                        & 47.07                        & 7.19                         \\
\rowcolor{mygray}Improvements &                      & {-1.40} & {-0.08} & {-3.27} & {-1.24} & {-2.89} & \red{1.18}  & {-0.02} & {-1.55} & \red{1.77}  & \red{3.38}  & \red{8.81}  & {-9.58}  & \red{3.31}  & \red{6.18}  & \red{0.34}  \\ \midrule
PIM          & Not         & 56.35                        & 71.06                        & 27.43                        & 43.42                        & 72.44                        & 8.13                         & 63.19                        & 80.34                        & 40.24                        & 47.75                        & 35.31                        & 50.85                         & 24.03                        & 42.59                        & 7.86                         \\
\textbf{+L-Reg}        &  minimized                    & 58.47                        & 76.49                        & 26.22                        & 44.20                        & 71.75                        & 10.85                        & 59.29                        & 77.96                        & 36.81                        & 49.74                        & 34.08                        & 50.01                         & 24.66                        & 43.86                        & 7.78                         \\
\rowcolor{mygray}Improvements &                      & \red{2.12}  & \red{5.43}  & {-1.21} & \red{0.78}  & {-0.70} & \red{2.72}  & {-3.90} & {-2.38} & {-3.43} & \red{2.00}  & {-1.23} & {-0.84}  & \red{0.63}  & \red{1.27}  & {-0.08} \\ \midrule
MIRO         & Minimized            & 56.83                        & 85.62                        & 24.85                        & 48.28                        & 80.61                        & 9.03                         & 61.53                        & 82.72                        & 35.03                        & 50.22                        & 39.92                        & 49.45                         & 31.49                        & 55.44                        & 10.57                        \\
\textbf{+L-Reg}        &                      & 68.44                        & 97.77                        & 25.64                        & 53.59                        & 79.50                        & 22.21                        & 62.07                        & 83.18                        & 35.21                        & 44.85                        & 40.87                        & 38.42                         & 31.58                        & 54.97                        & 10.98                        \\
\rowcolor{mygray}Improvements &                      & \red{11.61} & \red{12.14} & \red{0.79}  & \red{5.31}  & {-1.11} & \red{13.18} & \red{0.54}  & \red{0.46}  & \red{0.18}  & {-5.37} & \red{0.95}  & {-11.03} & \red{0.10}  & {-0.47} & \red{0.41}  \\ \midrule
GMDG         & Sufficiently  & 58.33                        & 91.46                        & 10.18                        & 48.85                        & 81.41                        & 9.22                         & 61.36                        & 83.31                        & 33.75                        &  40.02	& 32.38	 & 40.07 & {31.15} & 55.17                        & 10.18                        \\
\textbf{+L-Reg}        & minimized                     & 67.82                        & 91.86                        & 31.33                        & 51.96                        & 79.74                        & 18.15                        & 62.32                        & 82.77                        & 36.09                        &45.86& 39.77&	41.55                         & 31.75                        & 55.18                        & 11.30                        \\
\rowcolor{mygray}Improvements &                      & \red{9.50}  & \red{0.40}  & \red{21.15} & \red{3.11}  & {-1.68} & \red{8.92}  & \red{0.97}  & {-0.54} & \red{2.34}  & \red{5.83}	& \red{7.39}	& \red{1.49}  & \red{0.60}  & \red{0.01}  & \red{1.13} \\
\bottomrule
\end{tabular}%
}
\end{table}

\begin{table}[t]
\caption{MDG+GCD results: accuracy scores of each domain in PACS dataset.}
\label{tab:mdg_gcd_more_results1}
\resizebox{\linewidth}{!}{%
\begin{tabular}{l|ccc|ccc|ccc|ccc|ccc}
\toprule
PACS         & \multicolumn{3}{c|}{\textbf{Avg}}                                                          & \multicolumn{3}{c|}{\textbf{art\_painting}} & \multicolumn{3}{c|}{\textbf{cartoon}}      & \multicolumn{3}{c|}{\textbf{photo}}        & \multicolumn{3}{c}{\textbf{sketch}}        \\ 
Method      & {All}  & {Known}  & {Unknown} & {All}  & {Known}  & {Unknown} & {All}  & {Known}  & {Unknown} & {All}  & {Known}  & {Unknown}        & {All}  & {Known}  & {Unknown}         \\  \hline
ERM           & 57.26                        & 77.77                        & 22.33                        & 47.77         & 90.00        & 0.00         & 56.08        & 83.49        & 20.47        & 59.13        & 47.35        & 68.85        & 66.06        & 90.23        & 0.00         \\
with our reg & 55.86                        & 77.69                        & 19.06                        & 45.33         & 85.40        & 0.00         & 50.91        & 90.09        & 0.00         & 63.70        & 48.51        & 76.23        & 63.52        & 86.75        & 0.00         \\
Improvements & \blue{-1.40} & \blue{-0.08} & \blue{-3.27} & -2.44         & -4.60        & 0.00         & -5.17        & 6.60         & -20.47       & 4.57         & 1.16         & 7.38         & -2.54        & -3.48        & 0.00         \\ \hline
PIM          & 56.35                        & 71.06                        & 27.43                        & 46.80         & 55.17        & 37.32        & 50.37        & 89.15        & 0.00         & 62.05        & 49.50        & 72.40        & 66.19        & 90.40        & 0.00         \\
with our reg & 58.47                        & 76.49                        & 26.22                        & 46.74         & 88.05        & 0.00         & 56.50        & 78.77        & 27.57        & 64.30        & 48.51        & 77.32        & 66.35        & 90.62        & 0.00         \\
Improvements & \red{2.12}  & \red{5.43}  & \blue{-1.21} & -0.06         & 32.87        & -37.32       & 6.13         & -10.38       & 27.57        & 2.25         & -0.99        & 4.92         & 0.16         & 0.22         & 0.00         \\ \hline
MIRO         & 56.83                        & 85.62                        & 24.85                        & 51.86         & 97.70        & 0.00         & 56.45        & 99.91        & 0.00         & 48.35        & 75.17        & 26.23        & 70.64        & 69.72        & 73.16        \\
with our reg & 68.44                        & 97.77                        & 25.64                        & 68.46         & 97.82        & 35.24        & 61.51        & 98.02        & 14.09        & 72.60        & 98.84        & 50.96        & 71.18        & 96.39        & 2.26         \\
Improvements & \red{11.61} & \red{12.14} & \red{0.79}  & 16.60         & 0.11         & 35.24        & 5.06         & -1.89        & 14.09        & 24.25        & 23.68        & 24.73        & 0.54         & 26.67        & -70.90       \\ \hline
GMDG         & 58.33                        & 91.46                        & 10.18                        & 51.92         & 97.82        & 0.00         & 54.80        & 96.98        & 0.00         & 56.14        & 74.83        & 40.71        & 70.45        & 96.22        & 0.00         \\
with our reg & 67.82                        & 91.86                        & 31.33                        & 79.26         & 98.05        & 58.00        & 68.18        & 99.25        & 27.82        & 52.40        & 74.50        & 34.15        & 71.47        & 95.66        & 5.34         \\
Improvements & \red{9.50}  & \red{0.40}  & \red{21.15} & 27.33         & 0.23         & 58.00        & 13.38        & 2.26         & 27.82        & -3.74        & -0.33        & -6.56        & 1.02         & -0.56        & 5.34    \\
\bottomrule
\end{tabular}%
}
\end{table}

\begin{table}[t]
\caption{MDG+GCD results: accuracy scores of each domain in HomeOffice dataset.}
\label{tab:mdg_gcd_more_results2}
\resizebox{\linewidth}{!}{%
\begin{tabular}{l|ccc|ccc|ccc|ccc|ccc}
\toprule
HomeOffice   & \multicolumn{3}{c|}{\textbf{Avg}}                                                          & \multicolumn{3}{c|}{\textbf{Art}}          & \multicolumn{3}{c|}{\textbf{Clipart}}      & \multicolumn{3}{c|}{\textbf{Product}}      & \multicolumn{3}{c}{\textbf{Real World}}    \\ 
Method      & {All}  & {Known}  & {Unknown} & {All}  & {Known}  & {Unknown} & {All}  & {Known}  & {Unknown} & {All}  & {Known}  & {Unknown}        & {All}  & {Known}  & {Unknown}         \\  \hline
ERM           & 44.80                        & 74.67                        & 8.50                         & 45.26        & 72.68        & 3.26         & 37.94        & 64.48        & 10.19        & 46.71        & 78.74        & 9.87         & 49.28        & 82.80        & 10.68        \\
with our reg & 43.56                        & 71.78                        & 9.68                         & 41.30        & 62.72        & 8.47         & 35.91        & 60.78        & 9.90         & 48.34        & 78.95        & 13.14        & 48.68        & 84.67        & 7.22         \\
Improvements & \blue{-1.24} & \blue{-2.89} & \blue{1.18}  & -3.96        & -9.96        & 5.22         & -2.03        & -3.70        & -0.29        & 1.63         & 0.21         & 3.27         & -0.60        & 1.88         & -3.46        \\ \hline
PIM          & 43.42                        & 72.44                        & 8.13                         & 42.53        & 68.09        & 3.39         & 35.77        & 56.75        & 13.83        & 47.27        & 77.58        & 12.41        & 48.11        & 87.35        & 2.90         \\
with our reg & 44.20                        & 71.75                        & 10.85                        & 44.64        & 68.85        & 7.56         & 35.48        & 60.90        & 8.90         & 47.49        & 76.32        & 14.35        & 49.17        & 80.92        & 12.59        \\
Improvements & \red{0.78}  & \blue{-0.70} & \red{2.72}  & 2.11         & 0.77         & 4.17         & -0.29        & 4.15         & -4.92        & 0.23         & -1.26        & 1.94         & 1.06         & -6.43        & 9.69         \\ \hline
MIRO         & 48.28                        & 80.61                        & 9.03                         & 50.57        & 79.57        & 6.13         & 39.55        & 67.23        & 10.60        & 51.35        & 86.16        & 11.32        & 51.66        & 89.50        & 8.09         \\
with our reg & 53.59                        & 79.50                        & 22.21                        & 54.02        & 77.87        & 17.47        & 43.87        & 70.98        & 15.52        & 59.94        & 83.95        & 32.32        & 56.54        & 85.21        & 23.52        \\
Improvements & \red{5.31}  & \blue{-1.11} & \red{13.18} & 3.45         & -1.70        & 11.34        & 4.32         & 3.75         & 4.92         & 8.59         & -2.21        & 21.00        & 4.88         & -4.29        & 15.43        \\ \hline
GMDG         & 48.85                        & 81.41                        & 9.22                         & 51.60        & 81.96        & 5.08         & 40.89        & 69.30        & 11.19        & 51.15        & 87.53        & 9.32         & 51.75        & 86.87        & 11.30        \\
with our reg & 51.96                        & 79.74                        & 18.15                        & 52.83        & 79.15        & 12.52        & 43.59        & 69.02        & 16.99        & 56.31        & 83.11        & 25.48        & 55.11        & 87.67        & 17.59        \\
Improvements & \red{3.11}  & \blue{-1.68} & \red{8.92}  & 1.24         & -2.81        & 7.43         & 2.69         & -0.28        & 5.80         & 5.15         & -4.42        & 16.16        & 3.36         & 0.80         & 6.30     \\
\bottomrule
\end{tabular}%
}
\end{table}

\begin{table}[t]
\caption{MDG+GCD results: accuracy scores of each domain in VLCS dataset.}
\label{tab:mdg_gcd_more_results3}
\resizebox{\linewidth}{!}{%
\begin{tabular}{@{}l|ccc|ccc|ccc|ccc|ccc@{}} 
\toprule
VLCS         & \multicolumn{3}{c|}{\textbf{Avg}}                                                          & \multicolumn{3}{c|}{\textbf{Caltech101}}   & \multicolumn{3}{c|}{\textbf{LabelMe}}      & \multicolumn{3}{c|}{\textbf{SUN09}}        & \multicolumn{3}{c}{\textbf{VOC2007}}       \\ 
Method      & {All}  & {Known}  & {Unknown} & {All}  & {Known}  & {Unknown} & {All}  & {Known}  & {Unknown} & {All}  & {Known}  & {Unknown}        & {All}  & {Known}  & {Unknown}         \\  \hline
ERM           & 61.51                        & 82.89                        & 34.88                        & 82.07        & 74.87        & 85.85        & 50.54        & 92.01        & 4.85         & 62.07        & 95.15        & 11.38        & 51.35        & 69.51        & 37.45        \\
with our reg & 61.49                        & 81.33                        & 36.65                        & 76.59        & 75.13        & 77.36        & 50.64        & 91.02        & 6.13         & 60.70        & 92.83        & 11.48        & 58.02        & 66.35        & 51.63        \\
Improvements & \blue{-0.02} & \blue{-1.55} & \blue{1.77}  & -5.48        & 0.26         & -8.49        & 0.09         & -0.99        & 1.29         & -1.37        & -2.33        & 0.10         & 6.66         & -3.16        & 14.18        \\ \midrule
PIM          & 63.19                        & 80.34                        & 40.24                        & 80.39        & 72.05        & 84.77        & 53.84        & 91.74        & 12.07        & 62.22        & 94.21        & 13.21        & 56.31        & 63.36        & 50.92        \\
with our reg & 59.29                        & 77.96                        & 36.81                        & 72.61        & 73.33        & 72.24        & 53.98        & 90.75        & 13.45        & 56.85        & 83.64        & 15.81        & 53.72        & 64.13        & 45.75        \\
Improvements & \blue{-3.90} & \blue{-2.38} & \blue{-3.43} & -7.77        & 1.28         & -12.53       & 0.14         & -0.99        & 1.38         & -5.37        & -10.57       & 2.60         & -2.59        & 0.77         & -5.16        \\ \midrule
MIRO         & 61.53                        & 82.72                        & 35.03                        & 82.77        & 74.10        & 87.33        & 51.81        & 91.83        & 7.72         & 62.22        & 95.59        & 11.09        & 49.32        & 69.34        & 33.99        \\
with our reg & 62.07                        & 83.18                        & 35.21                        & 82.51        & 74.62        & 86.66        & 49.51        & 94.43        & 0.00         & 60.97        & 94.65        & 9.35         & 55.31        & 69.00        & 44.84        \\
Improvements & \red{0.54}  & \red{0.46}  & \red{0.18}  & -0.27        & 0.51         & -0.67        & -2.31        & 2.60         & -7.72        & -1.26        & -0.94        & -1.74        & 6.00         & -0.34        & 10.85        \\ \midrule
GMDG         & 61.36                        & 83.31                        & 33.75                        & 82.51        & 74.87        & 86.52        & 49.93        & 95.24        & 0.00         & 59.86        & 93.96        & 7.62         & 53.13        & 69.17        & 40.85        \\
with our reg & 62.32                        & 82.77                        & 36.09                        & 84.54        & 74.62        & 89.76        & 49.98        & 92.01        & 3.66         & 61.39        & 95.03        & 9.84         & 53.39        & 69.43        & 41.11        \\
Improvements & \red{0.97}  & \blue{-0.54} & \red{2.34}  & 2.03         & -0.26        & 3.23         & 0.05         & -3.23        & 3.66         & 1.52         & 1.07         & 2.22         & 0.26         & 0.26         & 0.26     \\
\bottomrule
\end{tabular}%
}
\end{table}

\begin{table}[t]
\caption{MDG+GCD results: accuracy scores of each domain in TerraIncognita dataset.}
\label{tab:mdg_gcd_more_results4}
\resizebox{\linewidth}{!}{%
\begin{tabular}{@{}l|ccc|ccc|ccc|ccc|ccc@{}}
\toprule
TerraIncognita     & \multicolumn{3}{c|}{\textbf{Avg}}                                                          & \multicolumn{3}{c|}{\textbf{Local 100}} & \multicolumn{3}{c|}{\textbf{Local 38}}      & \multicolumn{3}{c|}{\textbf{Local 43}}        & \multicolumn{3}{c}{\textbf{Local 46}}          \\ 
Method      & {All}  & {Known}  & {Unknown} & {All}  & {Known}  & {Unknown} & {All}  & {Known}  & {Unknown} & {All}  & {Known}  & {Unknown}        & {All}  & {Known}  & {Unknown}         \\  \hline
ERM           & 37.34                        & 20.46                        & 45.15                         & 46.51         & 1.25         & 57.07        & 39.88        & 28.22        & 44.91        & 29.41        & 24.65        & 40.25        & 33.59        & 27.70        & 38.36        \\
with our reg & 40.73                        & 29.27                        & 35.56                         & 52.94         & 0.14         & 65.27        & 39.26        & 33.76        & 41.64        & 40.24        & 57.45        & 1.03         & 30.47        & 25.71        & 34.32        \\
Improvements & \red{3.38}  & \red{8.81}  & \blue{-9.58}  & 6.43          & -1.11        & 8.20         & -0.62        & 5.53         & -3.27        & 10.83        & 32.80        & -39.22       & -3.12        & -1.99        & -4.04        \\ \midrule
PIM          & 47.75                        & 35.31                        & 50.85                         & 50.20         & 28.97        & 55.15        & 56.22        & 19.71        & 71.99        & 46.69        & 47.94        & 43.86        & 37.88        & 44.64        & 32.40        \\
with our reg & 49.74                        & 34.08                        & 50.01                         & 53.94         & 34.12        & 58.57        & 59.87        & 19.58        & 77.26        & 47.07        & 62.75        & 11.35        & 38.09        & 19.88        & 52.87        \\
Improvements & \blue{2.00}  & \blue{-1.23} & \blue{-0.84}  & 3.74          & 5.15         & 3.41         & 3.65         & -0.13        & 5.28         & 0.38         & 14.82        & -32.51       & 0.21         & -24.76       & 20.47        \\ \midrule
MIRO         & 50.22                        & 39.92                        & 49.45                         & 52.23         & 51.25        & 52.46        & 55.54        & 14.73        & 73.16        & 48.93        & 62.89        & 17.13        & 44.19        & 30.79        & 55.06        \\
with our reg & 44.85                        & 40.87                        & 38.42                         & 56.26         & 22.42        & 64.16        & 31.27        & 23.75        & 34.52        & 54.03        & 65.11        & 28.79        & 37.84        & 52.18        & 26.20        \\
Improvements & \blue{-5.37} & \red{0.95}  & \blue{-11.03} & 4.03          & -28.83       & 11.71        & -24.26       & 9.03         & -38.64       & 5.10         & 2.22         & 11.66        & -6.35        & 21.39        & -28.86       \\ \midrule
GMDG & 40.02 & 32.38 & 40.07 & 36.70 & 35.65 & 36.94 & 36.69 & 20.86 & 43.53 & 49.46 & 61.76 & 21.47 & 37.24 & 11.24 & 58.33 \\
with our reg & 45.86 & 39.77 & 41.55 & 51.89 & 38.16 & 55.09 & 41.30 & 22.05 & 49.61 & 50.47 & 65.29 & 16.72 & 39.77 & 33.59 & 44.79 \\
Improvements & \red{5.83} & \red{7.39} & \red{1.49} & 15.19 & 2.51 & 18.15 & 4.61 & 1.19 & 6.08 & 1.01 & 3.53 & -4.75 & 2.53 & 22.34 & -13.54 \\
\bottomrule
\end{tabular}%
}

\end{table}

\begin{table}[t]
\caption{MDG+GCD results: accuracy scores of each domain in DomainNet dataset.}
\label{tab:mdg_gcd_more_results5}
\resizebox{\linewidth}{!}{%
\begin{tabular}{@{}l|ccc|ccc|ccc|ccc}
\toprule

DomainNet    & \multicolumn{3}{c|}{\textbf{Avg}}                                                         & \multicolumn{3}{c|}{\textbf{clipart}}      & \multicolumn{3}{c|}{\textbf{info}}         & \multicolumn{3}{c}{\textbf{painting}}   
\\ Method      & {All}  & {Known}  & {Unknown} & {All}  & {Known}  & {Unknown} & {All}  & {Known}  & {Unknown} & {All}  & {Known}  & {Unknown}        \\ 
\hline
ERM           & 22.56                       & 40.89                        & 6.85                         & 31.04        & 58.32        & 7.15         & 17.94        & 34.71        & 6.85         & 30.59        & 51.82        & 9.34                \\
with our reg & 25.86                       & 47.07                        & 7.19                         & 32.03        & 58.43        & 8.91         & 18.17        & 34.31        & 7.50         & 31.93        & 52.58        & 11.24               \\
Improvements & \red{3.31} & \red{6.18}  & \red{0.34}  & 0.99         & 0.11         & 1.76         & 0.23         & -0.41        & 0.65         & 1.33         & 0.76         & 1.90            
\\ \midrule
PIM          & 24.03                       & 42.59                        & 7.86                         & 32.01        & 57.38        & 9.80         & 18.80        & 33.56        & 9.03         & 22.22        & 36.62        & 7.80               \\
with our reg & 24.66                       & 43.86                        & 7.78                         & 31.91        & 57.76        & 9.26         & 16.99        & 30.77        & 7.89         & 28.17        & 45.94        & 10.37             \\
Improvements & \red{0.63} & \red{1.27}  & \blue{-0.08} & -0.11        & 0.38         & -0.54        & -1.80        & -2.79        & -1.15        & 5.95         & 9.32         & 2.57             
\\ \midrule
MIRO         & 31.49                       & 55.44                        & 10.57                        & 40.13        & 67.55        & 16.11        & 25.84        & 48.53        & 10.84        & 37.89        & 62.45        & 13.29             \\
with our reg & 31.58                       & 54.97                        & 10.98                        & 40.61        & 66.72        & 17.75        & 25.58        & 45.83        & 12.19        & 36.74        & 62.29        & 11.15              \\
Improvements & \red{0.10} & \blue{-0.47} & \red{0.41}  & 0.49         & -0.83        & 1.64         & -0.26        & -2.70        & 1.35         & -1.15        & -0.16        & -2.14              
\\ \midrule
GMDG         & 31.15                       & 55.17                        & 10.18                        & 40.38        & 70.69        & 13.84        & 24.96        & 46.50        & 10.72        & 36.29        & 59.80        & 12.75               \\
with our reg & 31.75                       & 55.18                        & 11.30                        & 40.91        & 68.17        & 17.05        & 26.60        & 49.11        & 11.71        & 36.82        & 60.76        & 12.85              \\
Improvements & \red{0.60} & \red{0.01}  & \red{1.13}  & 0.53         & -2.52        & 3.21         & 1.63         & 2.61         & 0.99         & 0.53         & 0.96         & 0.10          \\
\toprule

DomainNet    & \multicolumn{3}{c|}{\textbf{Avg}}                                                       & \multicolumn{3}{c|}{\textbf{quickdraw}}    & \multicolumn{3}{c|}{\textbf{real}}      &
\multicolumn{3}{c}{\textbf{sketch}}
\\ Method      & {All}  & {Known}  & {Unknown} & {All}  & {Known}  & {Unknown} & {All}  & {Known}  & {Unknown} & {All}  & {Known}  & {Unknown}        \\ 
\hline
ERM           & -                       & -                       & -                          & 8.88         & 12.83        & 4.91         & 17.88        & 31.20        & 4.10         & 29.01        & 56.45        & 8.76         \\
with our reg& -                       & -                       & -                            & 9.04         & 14.73        & 3.31         & 34.34        & 63.94        & 3.69         & 29.68        & 58.41        & 8.49         \\
Improvements & -                       & -                       & -                           & 0.16         & 1.91         & -1.59        & 16.45        & 32.74        & -0.41        & 0.67         & 1.96         & -0.27        \\ \midrule
PIM         & -                       & -                       & -                            & 9.92         & 14.73        & 5.09         & 29.09        & 53.88        & 3.42         & 32.12        & 59.35        & 12.03        \\
with our reg & -                       & -                       & -                           & 9.94         & 15.11        & 4.74         & 30.26        & 56.13        & 3.47         & 30.68        & 57.43        & 10.95        \\
Improvements & -                       & -                       & -                           & 0.02         & 0.38         & -0.35        & 1.17         & 2.25         & 0.05         & -1.44        & -1.93        & -1.08        \\ \midrule
MIRO         & -                       & -                       & -                           & 8.06         & 12.12        & 3.98         & 42.19        & 75.49        & 7.72         & 34.83        & 66.51        & 11.46        \\
with our reg & -                       & -                       & -                           & 9.36         & 15.73        & 2.95         & 42.00        & 74.36        & 8.50         & 35.23        & 64.89        & 13.34        \\
Improvements& -                       & -                       & -                            & 1.30         & 3.61         & -1.03        & -0.20        & -1.14        & 0.78         & 0.40         & -1.62        & 1.88         \\ \midrule
GMDG         & -                       & -                       & -                           & 7.43         & 11.83        & 3.01         & 42.84        & 75.27        & 9.27         & 35.01        & 66.95        & 11.46        \\
with our reg & -                       & -                       & -                           & 9.11         & 13.51        & 4.70         & 42.63        & 74.42        & 9.72         & 34.44        & 65.13        & 11.80        \\
Improvements & -                       & -                       & -                           & 1.68         & 1.67         & 1.69         & -0.21        & -0.84        & 0.45         & -0.58        & -1.81        & 0.34  \\
\bottomrule
\end{tabular}%
}
\end{table}

\subsection{More GradCAM visualizations}
\label{app:vis}
We provide more visualized examples of L-Reg. Examples of known classes can be seen in \cref{fig:dog,fig:elephant,fig:giraffe,fig:guitar} and unknown classes in \cref{fig:horse,fig:person}. Compromises in known sets, as discussed in the limitations, can be seen in \cref{fig:elephant,fig:person}.

\begin{figure}[t]
  \centering
  \includegraphics[width=0.8\linewidth]{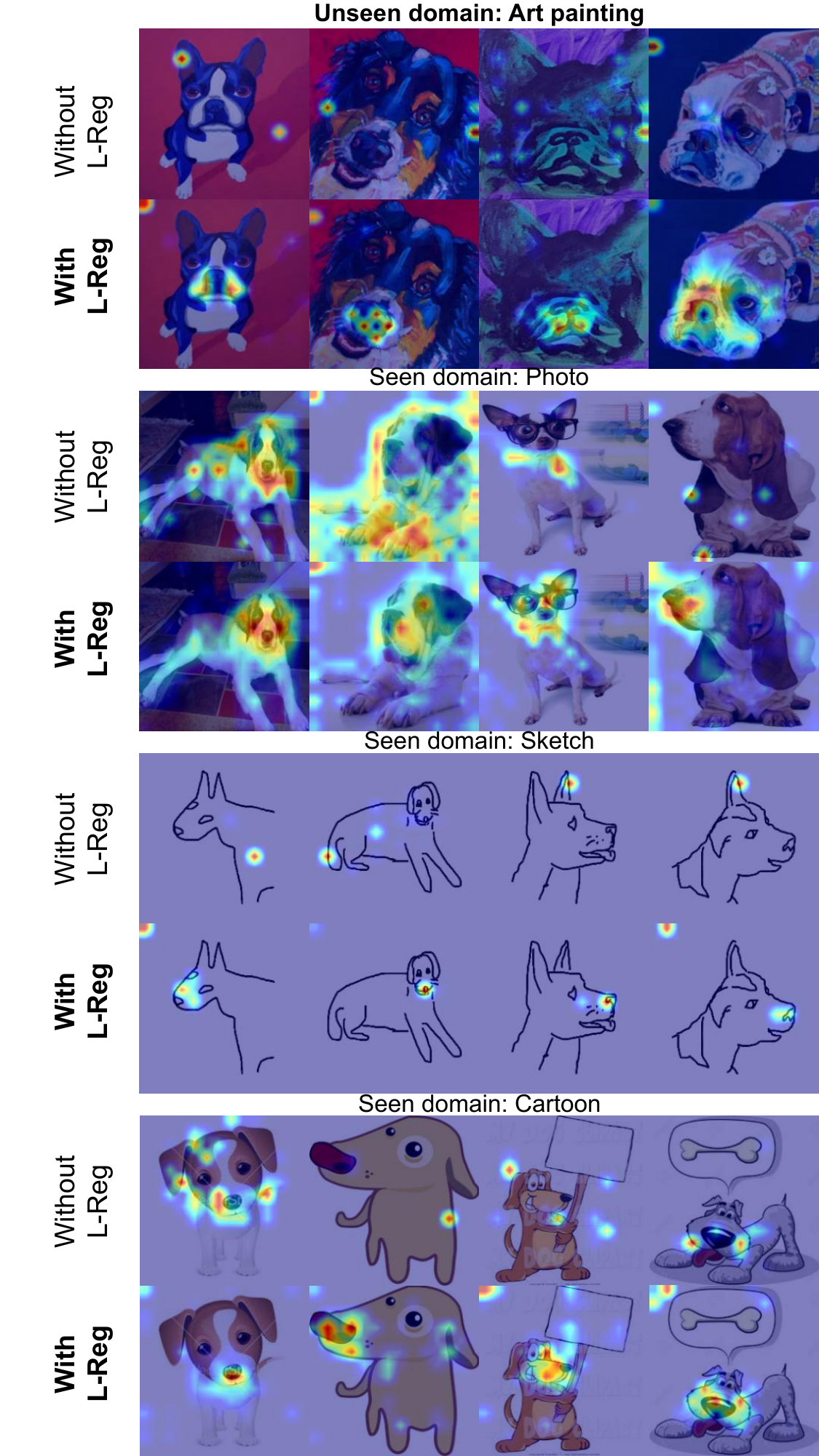}
  \caption{GradCAM visualizations: Baseline is GMDG. The used dataset is PACS. The model is trained under uDG+GCD setting with and without L-Reg, respectively.  It can be seen that for the \textbf{known} class `dog,' the model trained with L-Reg extracts the area around the nose area for classification across all seen and unseen domains.}
  \label{fig:dog}
\end{figure}

\begin{figure}[t]
  \centering
  \includegraphics[width=0.8\linewidth]{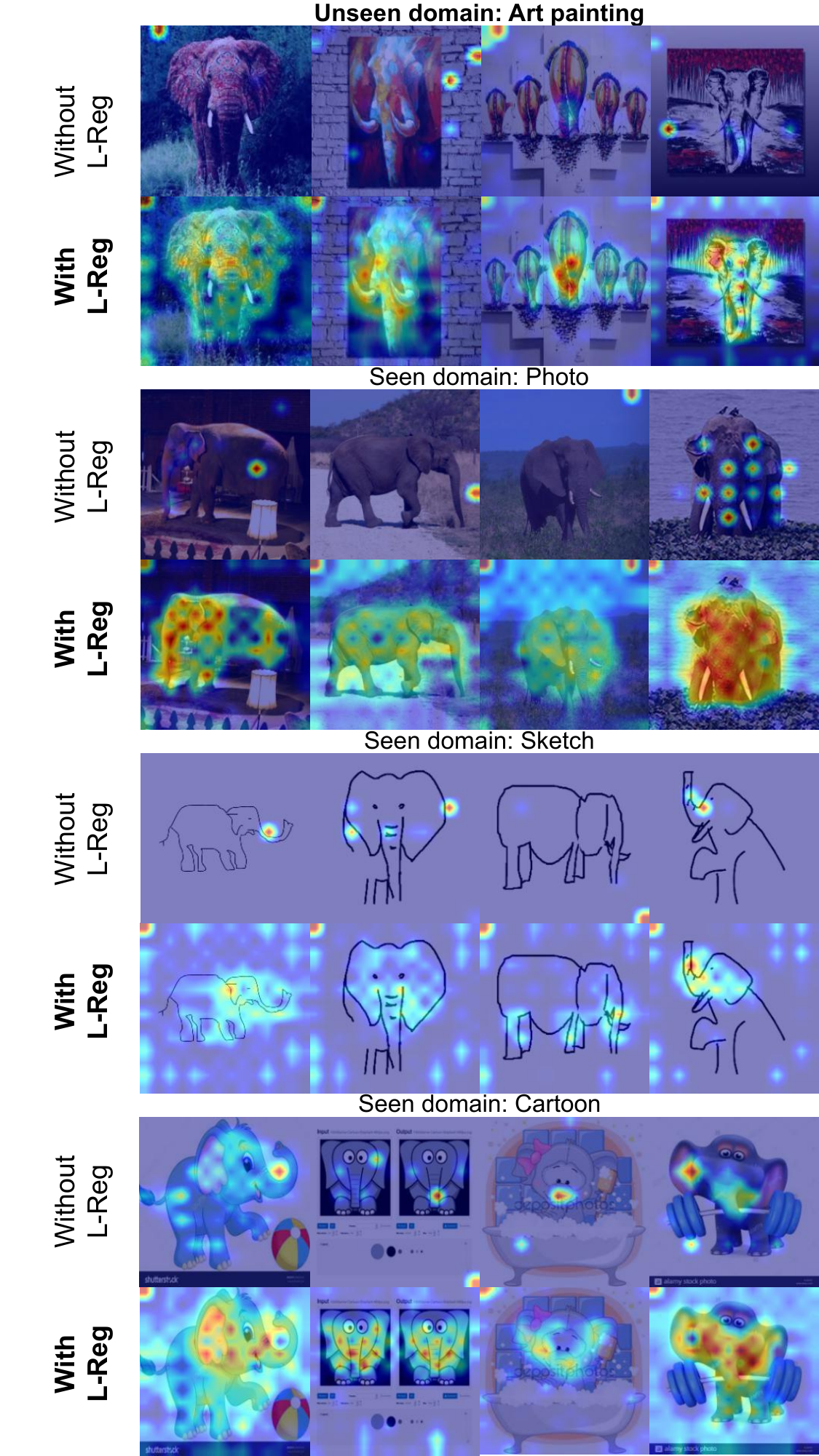}
  \caption{GradCAM visualizations: Baseline is GMDG. The used dataset is PACS. The model is trained under uDG+GCD setting with and without L-Reg, respectively.  It can be seen that for the \textbf{known} class `elephant,' the model trained with L-Reg extracts the shape of long noses, teeth, and big ears for classification across all seen and unseen domains. The compromise of the known sets can be seen in the sketch domain, where those features are not significant. }
  \label{fig:elephant}
\end{figure}

\begin{figure}[t]
  \centering
  \includegraphics[width=0.8\linewidth]{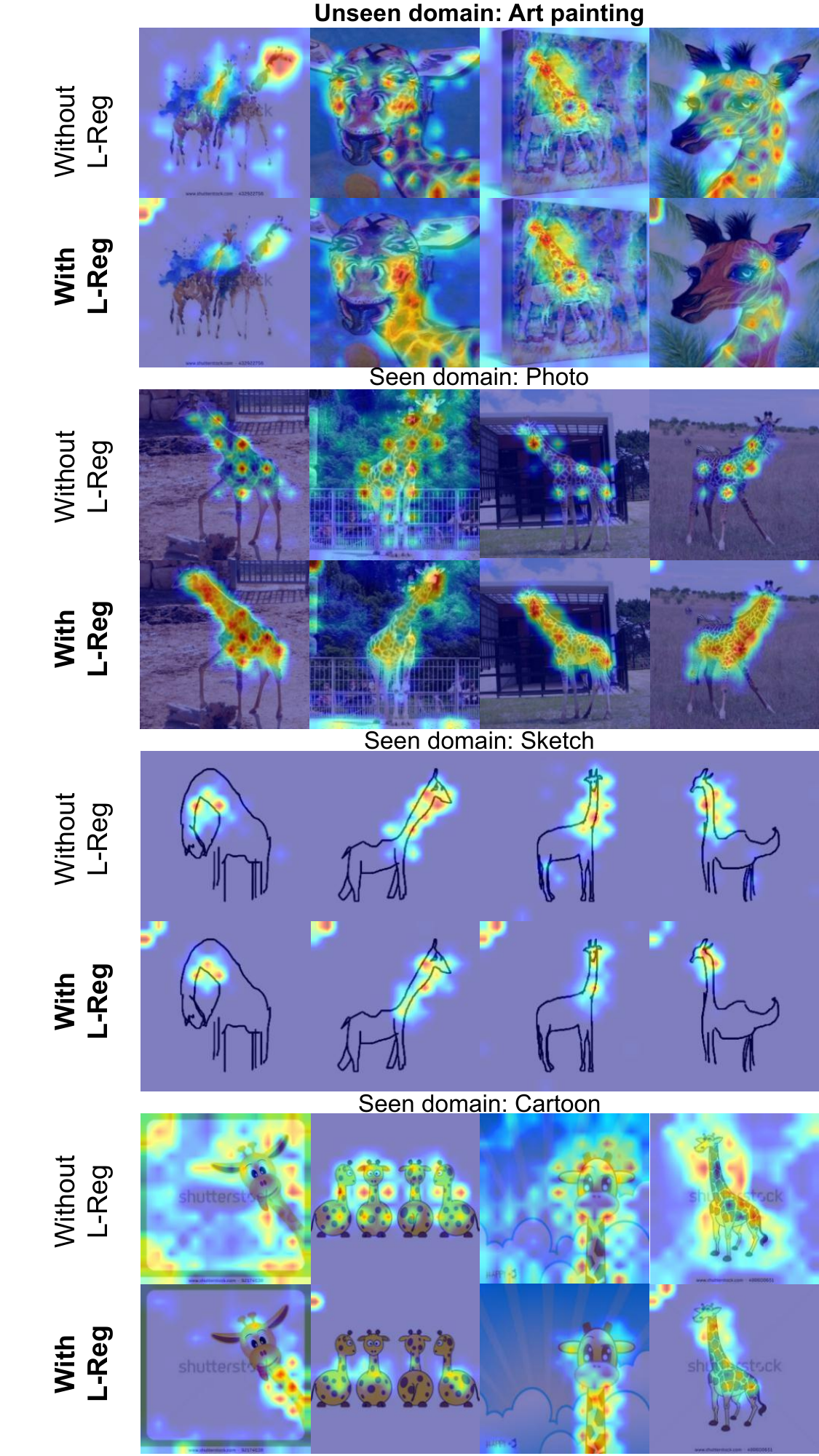}
  \caption{GradCAM visualizations: Baseline is GMDG. The used dataset is PACS. The model is trained under uDG+GCD setting with and without L-Reg, respectively.  It can be seen that for the \textbf{known} class `giraffe,' the model trained with L-Reg extracts the feature of the long necks for classifying across all seen and unseen domains.}
  \label{fig:giraffe}
\end{figure}

\begin{figure}[t]
  \centering
  \includegraphics[width=0.8\linewidth]{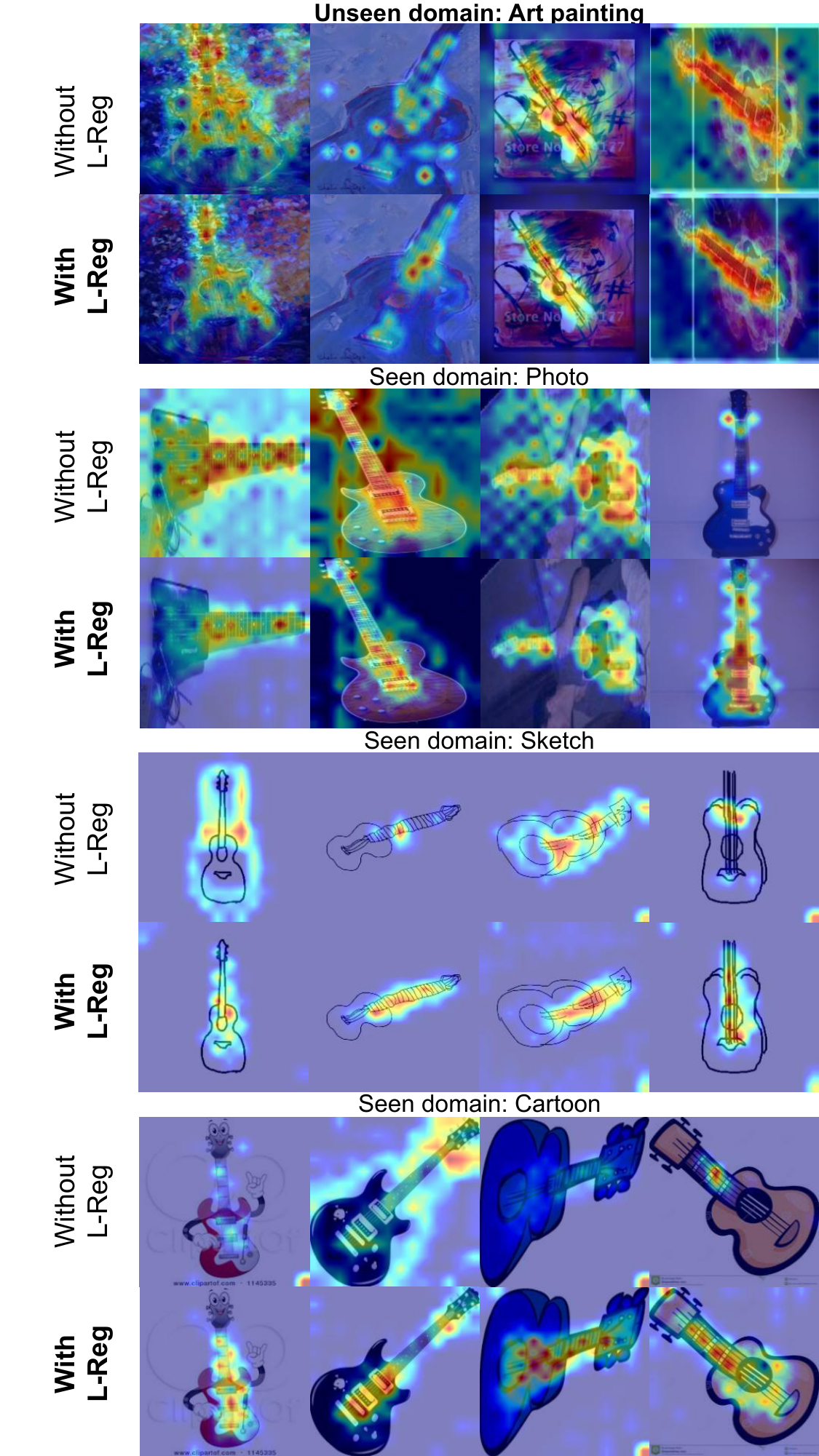}
  \caption{GradCAM visualizations: Baseline is GMDG. The used dataset is PACS. The model is trained under uDG+GCD setting with and without L-Reg, respectively.  It can be seen that for the \textbf{known} class `guitar,' the model trained with L-Reg extracts the features of the necks and the strings of the guitar for classification across all seen and unseen domains.}
  \label{fig:guitar}
\end{figure}

\begin{figure}[t]
  \centering
  \includegraphics[width=0.8\linewidth]{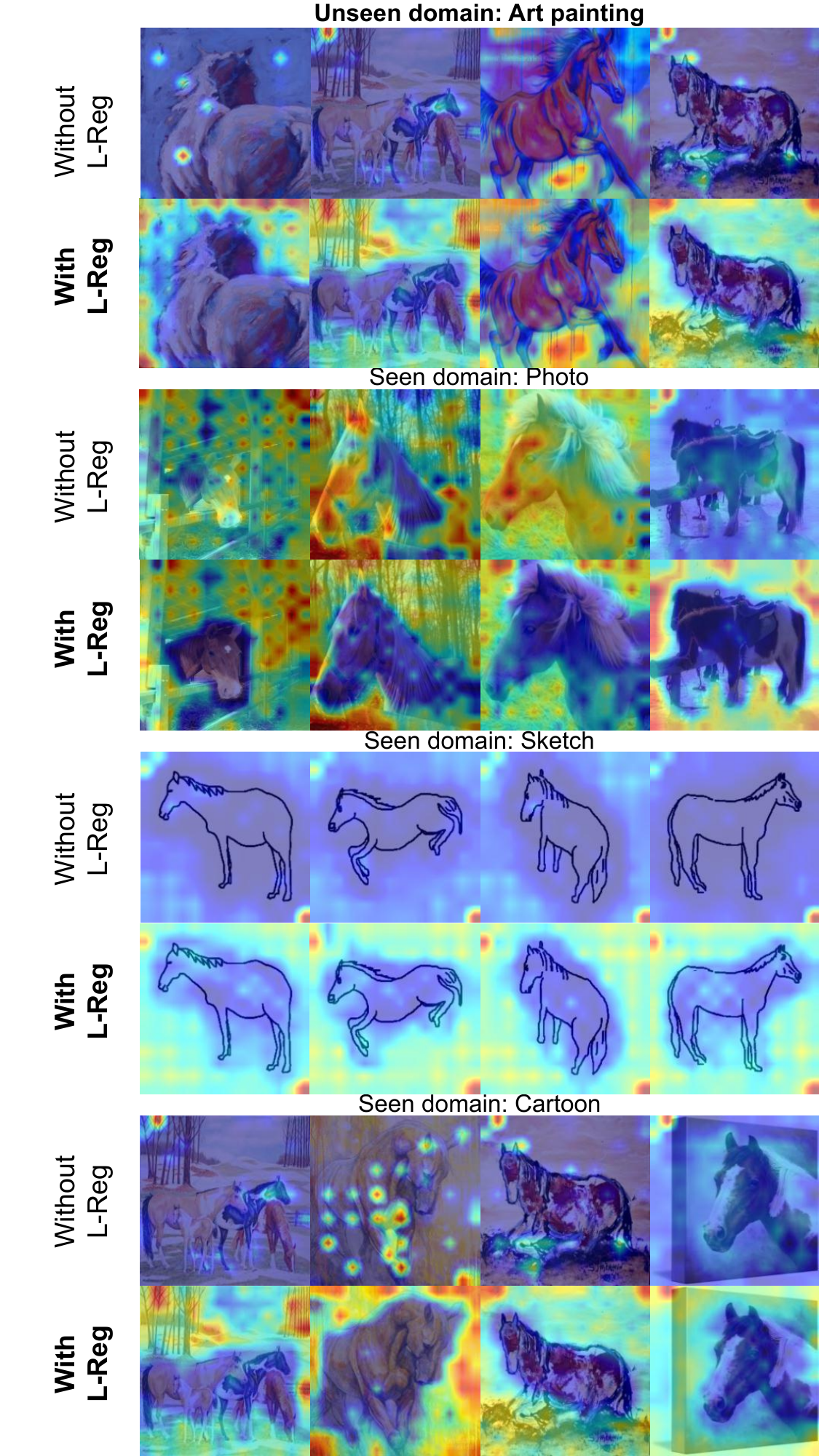}
  \caption{GradCAM visualizations: Baseline is GMDG. The used dataset is PACS. The model is trained under uDG+GCD setting with and without L-Reg, respectively.  It can be seen that for the \textbf{unknown} class `horse,' the model trained with L-Reg extracts the features of the overall outline shapes of horses for classification across all seen and unseen domains.}
  \label{fig:horse}
\end{figure}

\begin{figure}[t]
  \centering
  \includegraphics[width=0.8\linewidth]{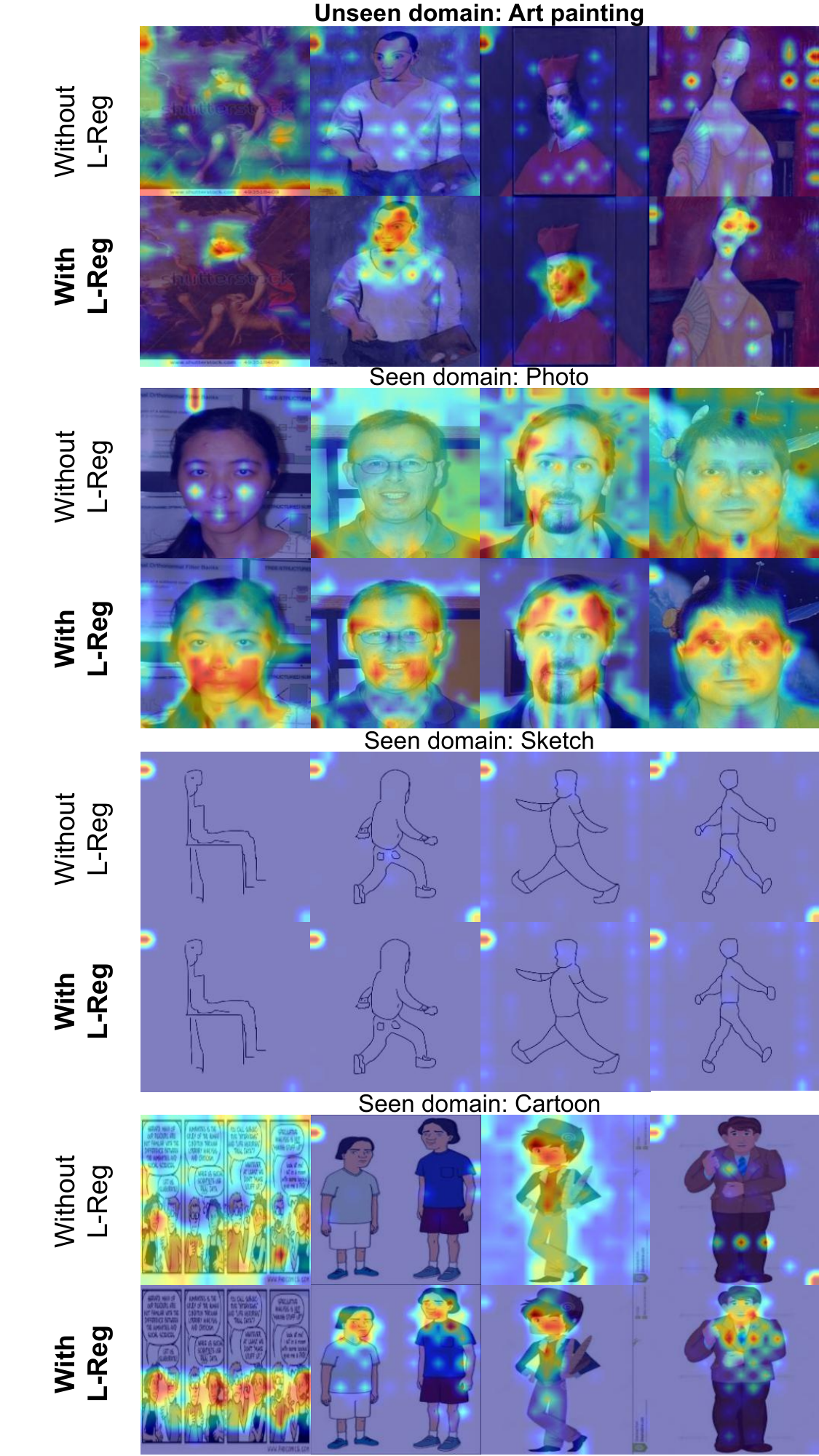}
  \caption{GradCAM visualizations: Baseline is GMDG. The used dataset is PACS. The model is trained under uDG+GCD setting with and without L-Reg, respectively.  It can be seen that for the \textbf{unknown} class `person,' the model trained with L-Reg extracts the features of human faces for classification across all seen and unseen domains. The compromise of the known sets can be seen in the sketch domain, where those faces are not drawn.}
  \label{fig:person}
\end{figure}

\end{document}